%% file: main.tex
\documentclass{article}
\usepackage{ijcai26}

\usepackage{times}
\usepackage{soul}
\usepackage{url}
\usepackage[utf8]{inputenc}
\usepackage[small]{caption}
\usepackage{graphicx}
\usepackage{xcolor}
\usepackage{listings}

\usepackage{amsmath, amsfonts, amsthm, amssymb}
\usepackage{mathtools}
\usepackage{bm}

\usepackage{booktabs}
\usepackage{multirow}
\usepackage{algorithm}
\usepackage{algorithmicx}
\usepackage{algpseudocode}

\usepackage{natbib}

\usepackage{inconsolata}
\usepackage[switch]{lineno}
\usepackage{enumitem}
\usepackage{tcolorbox}
\tcbuselibrary{skins,breakable,raster,listings}

\usepackage[hidelinks]{hyperref}

\input{math_commands.tex}

\lstdefinestyle{promptcode}{
  basicstyle=\ttfamily\footnotesize,
  breaklines=true,
  columns=fullflexible,
  keepspaces=true,
  showstringspaces=false,
  frame=single,
  framerule=0.35pt,
  rulecolor=\color{black!20},
  aboveskip=2pt,
  belowskip=2pt
}

\newcommand{\casehintbad}[1]{\vspace{-2pt}\par{\small\color{red!70!black}\textbf{Error:} #1}\vspace{4pt}}
\newcommand{\casehintgood}[1]{\vspace{-2pt}\par{\small\color{green!50!black}\textbf{OK:} #1}\vspace{4pt}}

\tcbset{
  casefixed/.style={
    height=6.2cm,
    valign=top,
    enlarge top by=0pt,
    enlarge bottom by=0pt,
  }
}

\newcommand{\ckpt}[1]{\textbf{\large #1}}

\lstdefinestyle{casecode}{
  basicstyle=\ttfamily\scriptsize,
  breaklines=true,
  breakatwhitespace=true,
  columns=fullflexible,
  keepspaces=true,
  showstringspaces=false,
  frame=single,
  framerule=0.35pt,
  rulecolor=\color{black!20},
  aboveskip=2pt,
  belowskip=2pt
}

\tcbset{
  casepanel/.style={
    colback=white,
    colframe=red!70!black,
    arc=6pt,
    boxrule=0.9pt,
    left=6pt,right=6pt,top=6pt,bottom=6pt,
    coltitle=white,
    colbacktitle=red!70!black,
    fonttitle=\small\bfseries,
  },
  casepanelBad/.style={
    colback=white,
    colframe=red!70!black,
    arc=6pt,
    boxrule=0.9pt,
    left=6pt,right=6pt,top=6pt,bottom=6pt,
    coltitle=white,
    colbacktitle=red!70!black,
    fonttitle=\small\bfseries,
  },
  casepanelGood/.style={
    colback=white,
    colframe=green!50!black,
    arc=6pt,
    boxrule=0.9pt,
    left=6pt,right=6pt,top=6pt,bottom=6pt,
    coltitle=white,
    colbacktitle=green!50!black,
    fonttitle=\small\bfseries,
  }
}

\theoremstyle{plain}
\newtheorem{theorem}{Theorem}[section]
\newtheorem{proposition}[theorem]{Proposition}
\newtheorem{lemma}[theorem]{Lemma}
\newtheorem{corollary}[theorem]{Corollary}

\theoremstyle{definition}

\theoremstyle{remark}
\newtheorem{remark}[theorem]{Remark}

\title{Beyond Hard Writes and Rigid Preservation: \\
Soft Recursive Least-Squares for Lifelong LLM Editing\textsuperscript{\S}}

\author{
\mbox{Xinyu Wang}$^{1,3,*}$
\and
\mbox{Sicheng Lyu}$^{1,2,3,*}$
\and
\mbox{Yu Gu}$^{1,*}$
\and
\mbox{Jerry Huang}$^{2,4}$
\and
\mbox{Peng Lu}$^4$\\
\mbox{Yufei Cui}$^1$
\and
\mbox{Xiao-Wen Chang}$^{1,\dagger}$\\
\affiliations
$^1$McGill University,
$^2$Mila--Quebec AI Institute,
$^3$SimpleWay.AI,
$^4$Universit\'e de Montr\'eal\\
\emails
\{xinyu.wang5, sicheng.lyu, yu.gu4\}@mail.mcgill.ca,
chang@cs.mcgill.ca
}

\begin{document}

\maketitle

\let\thefootnote\relax\footnotetext{\textsuperscript{\S}Accepted by IJCAI--ECAI 2026.}
\let\thefootnote\relax\footnotetext{$^*$Equal contribution.}
\let\thefootnote\relax\footnotetext{$^\dagger$Corresponding author.}

\input{sections/0_abstract}
\input{sections/1_introduction}
\input{sections/2_related_work}
\input{sections/3_methdology}
\input{sections/4_theory}
\input{sections/5_experiments}
\input{sections/6_conclusion}

\clearpage

\section*{Ethical Statement}

While our work centers on model-editing methods, it is important to acknowledge that such techniques can also be misused to inject undesirable knowledge or behavioral traits into a model. These risks merit careful consideration and discussion.


\bibliographystyle{abbrvnat}
\bibliography{ijcai26}

\clearpage

\appendix
\input{sections/x_appendix}

\end{document}

%% file: math_commands.tex
\usepackage{amsmath,amsfonts,bm}

\def\eqref#1{equation~\ref{#1}}

\def\1{\bm{1}}

\def\vk{{\bm{k}}}

\def\vv{{\bm{v}}}

\def\mA{{\bm{A}}}
\def\mB{{\bm{B}}}
\def\mC{{\bm{C}}}
\def\mD{{\bm{D}}}
\def\mE{{\bm{E}}}
\def\mF{{\bm{F}}}
\def\mG{{\bm{G}}}
\def\mH{{\bm{H}}}
\def\mI{{\bm{I}}}

\def\mK{{\bm{K}}}

\def\mM{{\bm{M}}}
\def\mN{{\bm{N}}}

\def\mP{{\bm{P}}}
\def\mQ{{\bm{Q}}}
\def\mR{{\bm{R}}}
\def\mS{{\bm{S}}}
\def\mT{{\bm{T}}}

\def\mV{{\bm{V}}}
\def\mW{{\bm{W}}}

\def\mY{{\bm{Y}}}

\DeclareMathAlphabet{\mathsfit}{\encodingdefault}{\sfdefault}{m}{sl}
\SetMathAlphabet{\mathsfit}{bold}{\encodingdefault}{\sfdefault}{bx}{n}

\newcommand{\E}{\mathbb{E}}

\newcommand{\R}{\mathbb{R}}

\DeclareMathOperator*{\argmin}{arg\,min}

%% file: sections/0_abstract.tex
\begin{abstract}
Model editing updates a pre-trained LLM with new facts or rules without retraining while preserving unrelated behavior.
In real deployment, edits arrive as long streams, creating a \emph{plasticity--stability dilemma}: repeated locate-then-edit ``hard writes'' can accumulate interference over time, while rigid preservation constraints may protect only explicitly constrained directions, allowing past edits or unconstrained behaviors to deviate.

We propose \textbf{RLSEdit}, a recursive least-squares editor for long sequential editing.
RLSEdit formulates editing as an online quadratic optimization with soft constraints, minimizing a cumulative key-value fitting objective together with two regularizers that control deviation from the pre-trained weights and from a designated anchor mapping.
This objective admits an efficient Woodbury-based online recursion, with per-edit cost independent of history length and scaling only with the current edit size.
We further provide deviation bounds and an asymptotic characterization of the adherence--preservation trade-off in the many-edits regime. Experiments on CounterFact and ZsRE across multiple model families show stable scaling to $10\mathsf{K}$ edits, outperforming strong baselines in both edit success and holistic stability, while retaining early edits and preserving general capabilities on GLUE and held-out reasoning/code benchmarks.

Code will be at \href{https://github.com/Euphoria040201/RLSEdit}{here}.
\end{abstract}

%% file: sections/1_introduction.tex
\section{Introduction}

Despite the large amount of knowledge they store within their parameters, large language models (LLMs)~\citep{qwen-2.5,gpt-4,deepseek-r1} inevitably contain outdated, incomplete, or incorrect knowledge~\citep{de-cao-etal-2021-editing,mitchell2022fast} when statically deployed without re-training or access to external knowledge bases. Due to the high computational cost of retraining from scratch, many applications require updating models through \emph{edits} to a subset of parameters, with the goal of integrating new facts or rules while preserving general model behavior~\citep{meng2022locating}.
While early model editors largely focused on single or small-batch updates, practical deployments are inherently sequential: edits arrive continuously, and the editor must remain reliable after each edit~\citep{DBLP:conf/nips/HartvigsenSPKG23,gupta-etal-2024-model}.

The many-edits regime presents a dilemma. To remain useful over long streams, an editor must both memorize incoming information and preserve knowledge acquired from earlier edits. In practice, failures often appear as two coupled forms of forgetting:
\begin{enumerate}[leftmargin=*]
    \item \emph{Retroactive Edit Forgetting}: future edits can overwrite edits applied in the past.
    \item \emph{General-Ability Degradation}: edits can deteriorate out-of-scope reasoning and language understanding.
\end{enumerate}

Existing sequential editors typically fail for complementary reasons.
\emph{Locate-then-edit} approaches~\citep{meng2022locating, meng2022mass} perform \emph{hard writes}, forcing the model to learn new associations with limited control over previously stored knowledge. As the number of edits accumulates, these updates may interfere with one another, leading to instability and retroactive forgetting of earlier facts. Conversely, \emph{null-space} editors~\citep{fang2410alphaedit, sun-etal-2025-mitigating-negative, lyu2026evoeditevolvingnullspacealignment} project updates into a feasible subspace conditioned on an anchor mapping, thereby preserving the associations explicitly constrained by that mapping. However, as edits are repeatedly applied, maintaining complete preservation would require a growing set of constraints, while loosening these constraints may still allow degradation in general abilities or retention of previous facts.
As a consequence, neither paradigm is sufficient for long sequential editing.

We focus on parameter-editing methods. This distinction is mechanism-based rather than dataset-based. ~\citep{wang2025resona}
Alternatively, sequential parameter editing can be viewed as \emph{online regularized least squares} on a layer-wise key-value surrogate~\citep{sayed2003fundamentals}. Each edit contributes a quadratic fitting term, while preservation is controlled by two explicit deviation terms: one penalizing deviation from the initial model, and another penalizing deviation from a designated \emph{anchor mapping}. This gives a soft-constraint formulation in which learning new edits and preserving previous knowledge are optimized within a single objective.

From this view, we propose \textbf{RLSEdit}, a recursive least-squares editor for long sequential editing. RLSEdit formulates editing as a quadratic objective and derives an efficient online recursion via the Woodbury identity~\citep{sherman1950adjustment, woodbury1950inverting, hager1989updating}. Its per-edit cost is independent of the number of previous edits and instead scales with the current edit rank/size. We further provide theoretical deviation bounds and an asymptotic characterization of the edit-preservation trade-off under long edit streams.

To summarize, our main contributions are:
\begin{itemize}[leftmargin=*]
    \item We formulate lifelong editing as online regularized least squares with explicit parameter-deviation and anchor-deviation controls, providing a soft-constraint alternative between hard writing and hard preservation.
    \item We derive a Woodbury-based online update whose per-edit cost is independent of the number of past edits.
    \item We provide deviation bounds and an asymptotic characterization of the adherence-preservation trade-off under long edit streams.
    \item We evaluate RLSEdit on \texttt{Llama-3}~\citep{llama3} and \texttt{Qwen2.5}~\citep{qwen-2.5} after $10\mathsf{K}$ edits, showing stronger edit success, improved early-edit retention, and better preservation of general capabilities on held-out reasoning and code benchmarks.
\end{itemize}

%% file: sections/2_related_work.tex
\section{Related Work}
\label{sec:related}

We review model editing methods through a \emph{soft versus hard constraint} lens. We explain how different editing methods balance between making new edits and preserving previous edits and knowledge, and why long-sequential editing is challenging. In particular, we consider layer-wise input-output pairs, where we edit a single linear map in layer $\ell$ (e.g., an attention projection). Given an input prompt, we run a forward pass and collect a set of module-level \emph{input-output} feature pairs $(\vk,\vv)$ at selected token positions, where $\vk\in\mathbb{R}^{d_k}$ is the input activation to the edited map and $\vv\in\mathbb{R}^{d_v}$ is the corresponding output activation. For the $t$-th edit, we stack $u_t$ such pairs to form $\mK_t\in\mathbb{R}^{u_t\times d_k}$ and $\mV_t\in\mathbb{R}^{u_t\times d_v}$.

\paragraph{The Writing and Preservation Trade-off.}
The goal of editing is to perform targeted updates to the model parameters to learn new key-value associations. With soft or no direct constraint on how this changes the parameters, this can be expressed as a constrained LS problem:
\begin{equation}
    \label{eq:rw_rome}
    \widehat{\mW}\ \in\ \underset{\mW}{\argmin}\ \|\mK\mW-\mV\|_F^2
    \quad \text{s.t.}\quad
    \mK^\star \mW = \mV^\star .
\end{equation}
Here $(\mK,\mV)$ denotes the current key-value associations contained within the model parameters, and $(\mK^\star,\mV^\star)$ denotes the edit constraints defined by the new set of edits.

Alternatively, one can attempt to preserve greater amounts of existing knowledge by restricting updates to a feasible subspace so that performance on a \emph{designated preservation set} (e.g., an anchor/background mapping) remains unchanged.
This can be expressed as
\begin{equation}
\label{eq:rw_null}
\resizebox{0.91\linewidth}{!}{$
\underset{\Delta_t}{\min} \|\mK_t(\mW_{t-1}\!+\!\Delta_t)\!-\!\mV_t\|_F^2 \; (+\lambda R(\cdot))
\;\text{s.t.}\;
\mK_{\mathrm{pres},t}\Delta_t\!=\!0.
$}
\end{equation}
This leads to a soft fit to the new edits under a hard preservation constraint. In \textsc{AlphaEdit}~\citep{fang2410alphaedit}, $\mK_{\mathrm{pres},t}$ is typically fixed to a chosen anchor/background set, preserving only what is explicitly constrained. \textsc{LangEdit}~\citep{sun-etal-2025-mitigating-negative} retains the same principle but updates multilingual preservation statistics online, so that $\mK_{\mathrm{pres},t}$ and the induced projector evolve over time.

\paragraph{Moving to Longer Edit Streams.}
When only a single batch of edits needs to be applied, existing editing methods have shown strong performance. However, such settings are not fully representative of real-world LLM use cases. Models often exist in environments where they are deployed for long periods of time and where the number of required edits can grow continuously. In such longer edit streams, existing methods can suffer for several reasons: hard satisfaction of each batch can induce growing interference between incoming edits and prior knowledge, leading to retroactive forgetting and degradation in out-of-scope performance, while attempting to preserve too much prior knowledge can lead to insufficient learning of new associations.

\paragraph{Soft Updates with Preservation.}
MEMIT~\citep{meng2022mass} uses a soft LS objective over a batch of past and new associations, but is not formulated as a long-stream sequential objective. To address the long-edit stream setting, we enforce \textit{soft} adherence and preservation within a quadratic objective that encompasses both hard regimes as limiting cases.

\paragraph{Relation to Continual Learning.}
RLSEdit is also related to continual learning methods such as EWC~\citep{kirkpatrick2017overcoming} and online Laplace~\citep{NEURIPS2018_f31b2046}, which use quadratic penalties to restrict parameter drift from previously learned solutions. In our objective, the term $\lambda^2\|\mW-\mW_0\|_F^2$ plays a similar regularizing role. However, our setting is different: RLSEdit targets long sequential LLM editing through layer-wise key-value constraints, combines the drift penalty with an anchor-mapping term, and admits an exact Woodbury-based recursive update whose per-edit cost is independent of the edit history length.

\paragraph{Beyond Direct Parameter Writes.}
Several recent lines differ from direct streaming parameter updates mainly in editing mechanism.
\textsc{AnyEdit}~\citep{jiang2025anyedit} broadens the scope of editable knowledge beyond simple factual statements, while \textsc{UnKE}~\citep{deng2024unke} targets unstructured knowledge.
For lifelong settings, \textsc{WISE}~\citep{wang2024wise} separates edited knowledge from pre-trained knowledge via dual memories and routing, and \textsc{RECIPE}~\citep{chen2024lifelong} externalizes updates as retrieval-augmented continuous prompts with gating. 
These approaches can be evaluated on related editing tasks, but they intervene through additional components such as memory, retrieval, or prompting. They are therefore complementary to our focus on an efficient, single-objective streaming parameter editor. Model editing is also related to continual learning, especially quadratic-regularization methods such as EWC and online Laplace approximations~\citep{kirkpatrick2017overcoming, NEURIPS2018_f31b2046}. 
However, RLSEdit is designed for post-training LLM editing and couples cumulative key--value least squares, an anchor-mapping penalty, and a Woodbury recursion with history-independent per-edit cost.

%% file: sections/3_methdology.tex
\section{Methodology}
We present our editing framework in three parts.
We first introduce a recursive least-squares (RLS) formulation that accumulates all editing residuals in a single quadratic objective, with penalties on both deviation from the initial model parameters and deviation from an anchor mapping in Section~\ref{subsec:rls_editor}.
We then analyze the computational complexity of the resulting updates and compare them with existing sequential editors under long edit streams in Section~\ref{subsec:complexity}.
Finally, we contrast \emph{soft} and \emph{hard} constraint designs, showing how existing editors can be viewed as limiting cases of our formulation in Section~\ref{subsec:comparison_nullspace}.

\subsection{A recursive least squares editor}
\label{subsec:rls_editor}

\paragraph{Setup}
We consider editing a single linear map $\mW\!\in\!\mathbb{R}^{d_k{\times}d_v}$, such as a projection matrix in a transformer layer.
Each edit provides a set of key-value constraints
$(\mK_t,\mV_t)$, where $\mK_t\!\in\!\mathbb{R}^{u_t\times d_k}$ and $\mV_t\!\in\!\mathbb{R}^{u_t\times d_v}$, with $u_t$ being the number of contexts collected for the $t$-th edit.
Layer-wise edit adherence is measured by the residual $\left\|\mK_t\mW-\mV_t\right\|_F$.
The goal of our method is two-fold. First, it should incorporate \emph{all} edits up to time $t$. Second, it should control deviation from the initial weights $\mW_0$ and from a set of anchor pairs $(\mK_0,\mV_0)$. We therefore introduce two regularization terms for these two purposes.

\subsubsection{Regularized Least-Squares Equation}
At time $t$, our objective is to obtain the optimal weight $\mW_t^*$ that minimizes:
\begin{equation}
    \begin{split}
        \label{eq:LS_main}
    \mW_t^* \coloneq \underset{\mW}{\argmin} \sum_{i=1}^{t} & \left\|\mK_i \mW - \mV_i\right\|_F^2 \\
    + \lambda^2 & \left\|\mW-\mW_0\right\|_F^2 \\
    + \mu^2 & \left\|\mK_0 \mW - \mV_0\right\|_F^2,
    \end{split}
\end{equation}
where $\lambda$ and $\mu$ are hyperparameters.
The first term fits all edit pairs observed so far.
The second term penalizes parameter drift from the initial weight $\mW_0$, and the third term encourages the edited weight to preserve the anchor mapping $\mK_0\mW\approx\mV_0$.

To rewrite Equation \ref{eq:LS_main} as a standard matrix least-squares problem, define
\begin{equation}
    \begin{aligned}
        \mA_t
        &=
        \begin{bmatrix}
        \lambda \mI_{d_k} \\
        \mu \mK_0 \\
        \mK_1 \\
        \vdots \\
        \mK_t
        \end{bmatrix},
        &
        \mB_t
        &=
        \begin{bmatrix}
        \lambda \mW_0 \\
        \mu \mV_0 \\
        \mV_1 \\
        \vdots \\
        \mV_t
        \end{bmatrix}.
    \end{aligned}
\end{equation}
Then Equation \ref{eq:LS_main} becomes
\begin{equation}
    \label{eq:LSP_main}
    \mW_t^* = \underset{\mW}{\argmin} \left\|\mA_t\mW - \mB_t\right\|_F^2.
\end{equation}

To make the closed-form solution explicit, consider one output column at a time.
For a column $w$ of $\mW$ and the corresponding column $b_t$ of $\mB_t$, Equation \ref{eq:LSP_main} reduces to
\[
    \min_w \|\mA_t w - b_t\|_2^2.
\]
Taking the derivative and setting it to zero gives the normal equation
\[
    \mA_t^\top \mA_t w_t^* = \mA_t^\top b_t.
\]
Stacking the solutions for all output columns gives the matrix normal equation
\[
    \mA_t^\top \mA_t \mW_t^* = \mA_t^\top \mB_t.
\]
Let
\begin{align}
    \mS_t &\coloneq \mA_t^\top \mA_t
    = \lambda^2 \mI + \mu^2 \mK_0^\top \mK_0 + \sum_{i=1}^t \mK_i^\top \mK_i, 
    \label{eq:closed_form_s}
    \\
    \mT_t &\coloneq \mA_t^\top \mB_t
    = \lambda^2 \mW_0 + \mu^2 \mK_0^\top \mV_0 + \sum_{i=1}^t \mK_i^\top \mV_i.
    \label{eq:closed_form_t}
\end{align}
When $\lambda > 0$, $\mS_t\succ 0$, so the minimizer is unique and is given by
\begin{equation}
    \label{eq:closed_form}
    \mW_t^* = \mS_t^{-1} \mT_t.
\end{equation}
This solution jointly fits all edit pairs up to time $t$, while the two regularizers keep the edited weights close to $\mW_0$ and preserve the anchor mapping $\mK_0\mW\approx\mV_0$.

\begin{figure}[t]
    \centering
    \includegraphics[width=\linewidth]{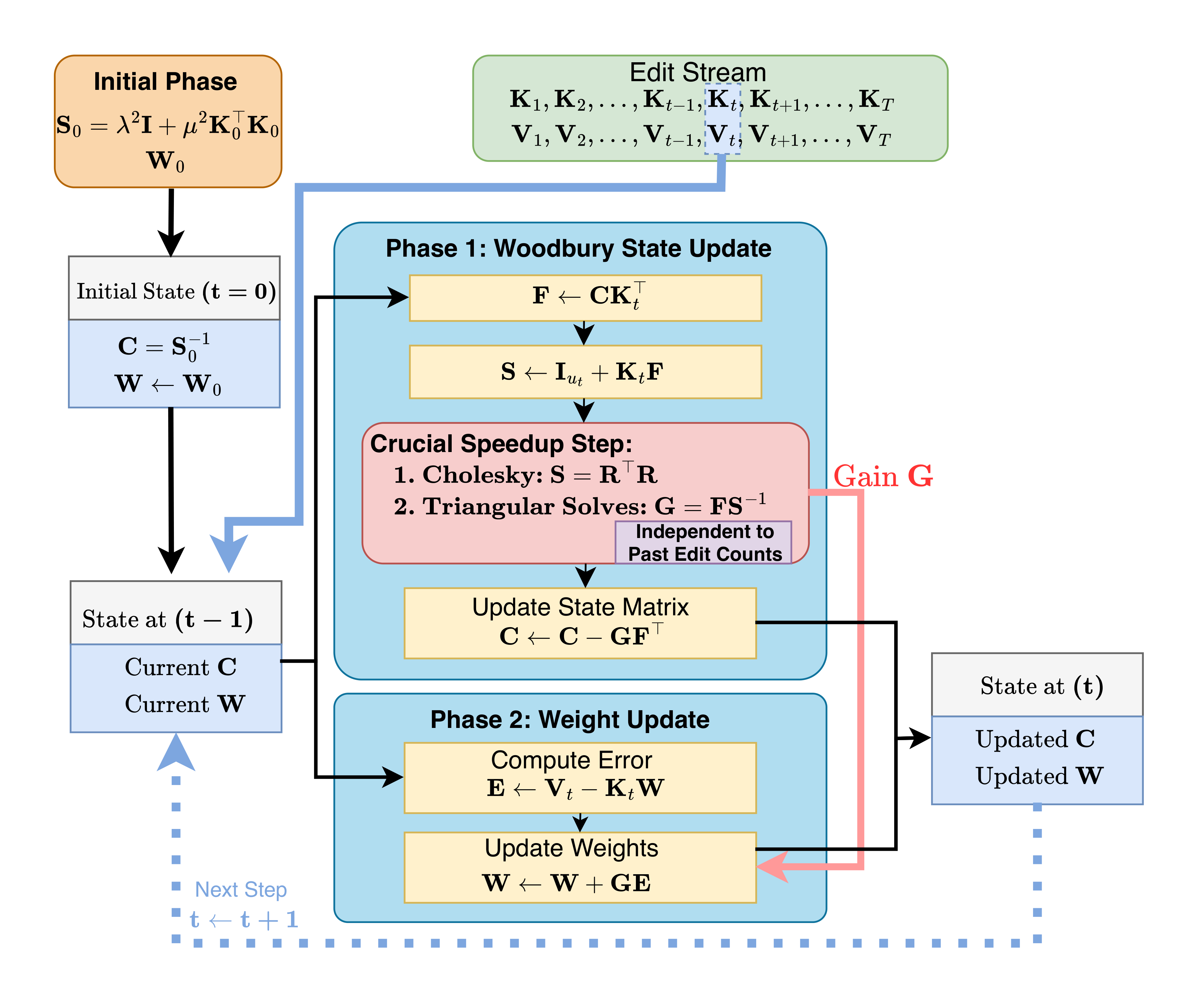} 
    \caption{The recursive workflow of our RLS-Woodbury editor. The process alternates between updating the covariance state via the Woodbury identity and updating the weights. The highlighted block shows how the update avoids a full $O(d_k^3)$ factorization during the edit stream by solving only small $u_t\times u_t$ systems.}
    \label{fig:rls_workflow}
\end{figure}

\subsubsection{Efficient Recursion via Normal Equations}
\label{subsubsec:normal_eq}

Direct computation of~Equation \ref{eq:closed_form} requires inverting $\mS_t$, which is expensive when repeated over a long edit stream.
We therefore derive an efficient recursive update.
From Equation \ref{eq:closed_form}, the minimizer $\mW_t^*$ satisfies
\begin{equation}
    \label{eq:NE_main}
    \left(\mA_t^\top \mA_t\right)\mW_t^* = \mA_t^\top \mB_t.
\end{equation}
Let $\mC_t\coloneq \mS_t^{-1}$.
Using Equation \ref{eq:closed_form_s}, we have
\begin{equation}
    \label{eq:C_inverse_update}
    \mC_t^{-1} = \mC_{t-1}^{-1} + \mK_t^\top \mK_t.
\end{equation}
Next, define
\[
    \mF_t \coloneq \mC_{t-1}\mK_t^\top \in \mathbb{R}^{d_k\times u_t},
    \qquad
    \mH_t \coloneq \mI_{u_t}+\mK_t\mF_t .
\]
By the Sherman--Morrison--Woodbury identity,
\begin{equation}
    \label{eq:woodbury_main}
    \mC_t
    =
    \mC_{t-1}
    -
    \mF_t \mH_t^{-1} \mF_t^\top .
\end{equation}
In implementation, we factorize $\mH_t=\mR_t^\top\mR_t$ by Cholesky and use triangular solves, avoiding explicit matrix inversion.

From Equation \ref{eq:closed_form_t}, we also have
\[
    \mT_t = \mT_{t-1} + \mK_t^\top \mV_t .
\]
Combining this with the covariance recursion gives the weight update
\begin{equation}
    \label{eq:W_recursion_main}
       \mW_t^* = \mW_{t-1}^* + \mC_t\mK_t^\top\left(\mV_t - \mK_t\mW_{t-1}^*\right).
\end{equation}
Equivalently, since $\mC_t\mK_t^\top=\mF_t\mH_t^{-1}$, the update only requires solving the small $u_t\times u_t$ system in $\mH_t$.
Letting $\mE_t \coloneq \mV_t-\mK_t\mW_{t-1}^*$, each edit requires only updating $\mC_t$ via Equation \ref{eq:woodbury_main} and then updating $\mW_t^*$ via Equation \ref{eq:W_recursion_main}.

\begin{algorithm}[t]
    \caption{RLS-Woodbury Editing}
    \label{alg:rls_woodbury}
    
    \begin{algorithmic}[1]
        \Require Initial weight $\mW_0$; 
        anchor pair $(\mK_0,\mV_0)$;
        penalties $(\lambda, \mu)$;
        edit stream $\{(\mK_t,\mV_t)\}_{t=1}^T$.
        \Ensure Edited weight $\mW_T^*$.
        
        \State $\mS_0 \gets \lambda^2 \mI_{d_k} + \mu^2 \mK_0^\top \mK_0$
        \State $\mT_0 \gets \lambda^2 \mW_0 + \mu^2 \mK_0^\top \mV_0$
        \State $\mS_0 = \mR_0^\top \mR_0$ \Comment{Cholesky factorization}
        \State $\mC_0 \gets \mS_0^{-1}$ \Comment{Computed via triangular solves}
        \State $\mW_0^* \gets \mC_0 \mT_0$
        
        \For{$t=1,2,\dots,T$}
            \Statex \hspace*{\algorithmicindent} Covariance update
            \State $\mF_t \gets \mC_{t-1}\mK_t^\top$
            \State $\mH_t \gets \mI_{u_t} + \mK_t \mF_t$
            \State $\mH_t = \mR_t^\top \mR_t$ \Comment{Cholesky factorization}
            \State $\mY_t \gets \mF_t \mR_t^{-1}$ \Comment{Triangular solve}
            \State $\mC_t \gets \mC_{t-1} - \mY_t \mY_t^\top$
        
            \Statex \hspace*{\algorithmicindent} Weight update
            \State $\mE_t \gets \mV_t - \mK_t \mW_{t-1}^*$
            \State $\mG_t \gets \mF_t \mH_t^{-1}$ \Comment{Gain matrix}
            \State $\mW_t^* \gets \mW_{t-1}^* + \mG_t \mE_t$
        \EndFor
        
        \State \Return $\mW_T^*$
    \end{algorithmic}
\end{algorithm}

\subsection{Complexity analysis}
\label{subsec:complexity}

We report the per-edit cost at step $t$.
Multiplying $\mM\in\R^{m\times n}$ and $\mN\in\R^{n\times p}$ costs $O(mnp)$, and solving a dense $n\times n$ linear system costs $O(n^3)$.

\subsubsection{RLS-Woodbury Updates.}
RLSEdit maintains $\mC_t=\mS_t^{-1}\in\R^{d_k\times d_k}$ and updates it via Woodbury using
\[
    \mF_t=\mC_{t-1}\mK_t^\top\in\mathbb{R}^{d_k\times u_t},
    \qquad
    \mH_t=\mI_{u_t}+\mK_t\mF_t\in\mathbb{R}^{u_t\times u_t}.
\]
The covariance-state update is dominated by forming these products, solving the resulting $u_t\times u_t$ system, and applying the low-rank correction:
\[
\text{Covariance update:}\qquad
O\!\left(d_k^2u_t + d_ku_t^2 + u_t^3\right).
\]
For the weight update, we reuse the same $u_t\times u_t$ solve to apply the gain
$\mG_t=\mC_t\mK_t^\top=\mF_t\mH_t^{-1}\in\R^{d_k\times u_t}$ and update $\mW_t$ using the residual $\mE_t$.
This step is dominated by the key-value multiplication against $d_v$ outputs:
\[
\text{Weight update:}\qquad
O\!\left(d_kd_vu_t + d_ku_t^2\right).
\]
Overall, the per-edit runtime is therefore
\[
\text{Per edit:}\qquad
O\!\left(d_k^2u_t + d_kd_vu_t + d_ku_t^2 + u_t^3\right),
\]
which simplifies to $O\!\left(d_k^2u_t + d_kd_vu_t\right)$ when $u_t\ll d_k,d_v$.

\subsubsection{Comparison to other sequential editors.}
For a fair long-sequential comparison, we focus on existing sequential editors.
\textsc{AlphaEdit} introduces \emph{hard preservation} by projecting the weight update onto the null space of a fixed preserved-knowledge set.
In sequential editing, it also regularizes against disrupting previously updated knowledge represented by accumulated key-value pairs.
A direct sequential implementation therefore involves history-dependent dense matrices over the feature dimension and repeated projection or factorization steps.
Let $m_{t-1}$ denote the number of previously updated pairs and let $u_t$ denote the number of pairs in the current edit.
The dominant cost can include a dense $d_k\times d_k$ factorization, together with Gram matrix construction over the accumulated pairs:
\[
\text{Per edit:}\qquad
O(d_k^3)\;+\;O\!\left(d_k^2(m_{t-1}+u_t)\right).
\]
Even with more careful incremental implementations, the cost remains tied to the accumulated preservation set or the dimension of the remaining feasible subspace.
In contrast, RLSEdit avoids an $O(d_k^3)$ factorization during the edit stream by maintaining $\mC_t=(\mA_t^\top\mA_t)^{-1}$ and using a Woodbury recursion.
Each edit only solves a $u_t\times u_t$ system, where typically $u_t\ll d_k$, making the update more efficient in long-edit settings, as shown in \ref{tab:efficiency_update_time}.

\subsection{Hard versus Soft Constraints}
\label{subsec:comparison_nullspace}
RLSEdit can be viewed as a soft-constraint alternative to two common hard regimes. Locate-then-edit methods such as ROME enforce the new association as a hard write while softly penalizing background deviation. Null-space methods such as \textsc{AlphaEdit} instead enforce hard preservation by restricting updates to a feasible subspace. In contrast, RLSEdit keeps both edit adherence and preservation soft through the cumulative objective
\[
\begin{aligned}
\mW_t^\star
= \arg\min_{\mW}\;&
\sum_{i=1}^{t}\|\mK_i\mW-\mV_i\|_F^2 + \lambda^2\|\mW-\mW_0\|_F^2 \\
&+ \mu^2\|\mK_0\mW-\mV_0\|_F^2 .
\end{aligned}
\]
As $\mu\to\infty$, the anchor mapping becomes a hard constraint. Similarly, assigning infinite weight to selected past fitting terms recovers hard preservation through the standard penalty method. Thus, RLSEdit interpolates between hard-write and hard-preserve behavior while remaining in a soft--soft regime for finite $(\lambda,\mu)$.

%% file: sections/4_theory.tex
\section{Theoretical Analysis}
\label{subsec:theory}

We provide deviation bounds in terms of $(\lambda,\mu)$:
$\lambda$ controls global parameter deviation from $\mW_0$, and $\mu$ controls deviation of the anchor mapping $\mK_0\mW$.

\begin{figure}
    \centering
    \includegraphics[width=\linewidth]{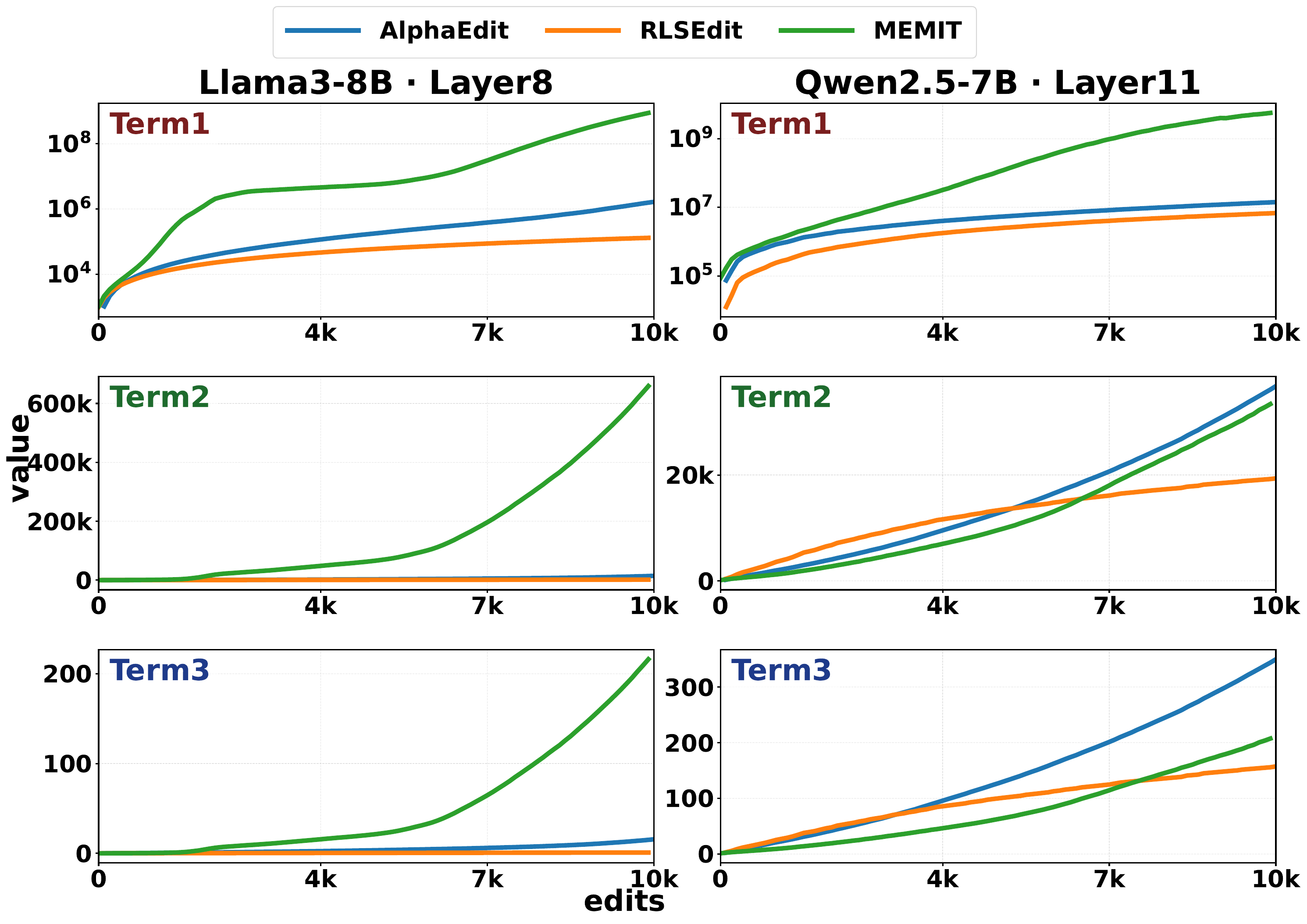}
    \caption{\textbf{Evolution of objective terms over $10\mathsf{K}$ edits.}
    We compare \textsc{RLSEdit} against baselines (\textsc{AlphaEdit}, \textsc{MEMIT}) on three metrics:
    \textbf{Term 1} ($\|\mK_t \mW-\mV_t\|_F^2$) measures the fitting error for the current edit;
    \textbf{Term 2} ($\|\mW-\mW_0\|_F^2$) measures parameter drift from the initial weights; and
    \textbf{Term 3} ($\|\mK_0 \mW-\mV_0\|_F^2$) measures the preservation error on the anchor mapping.
    The results show that \textsc{RLSEdit} maintains lower values across all three terms, supporting the stability of our soft-constraint formulation.}
    \label{fig:terms_decomposition}
\end{figure}

\begin{theorem}[Global deviation bounds]
\label{prop:global_deviation}
Let $\mW_t^*$ be the minimizer of $J_t(\mW)$ and define
$\mR_t\coloneq\mV_t-\mK_t\mW_{t-1}^*$. Let $\sigma_{\min}(\mK)$ denote the smallest singular value of $\mK$.

\begin{enumerate}[label=(\roman*),leftmargin=*]
\item (\emph{Parameter deviation}) If $\lambda>0$, then for any $T\ge 1$,
\[
\left\|\mW_T^*-\mW_0\right\|_F
\le
\frac{1}{\lambda^2}
\left\|
\sum_{t=1}^T \mK_t^\top
\left(\mV_t-\mK_t\mW_0\right)
\right\|_F .
\]

\item (\emph{Anchor-map deviation}) If $\mu>0$, then for any $T\ge 1$,
\[
\left\|\mK_0(\mW_T^*-\mW_0)\right\|_F
\le
\frac{1}{\mu}\sum_{t=1}^T \left\|\mR_t\right\|_F .
\]
\end{enumerate}
\end{theorem}

Theorem \ref{prop:global_deviation} shows that $\lambda$ and $\mu$ control different types of deviation.
A larger $\lambda$ keeps the solution closer to the original weights $\mW_0$, while a larger $\mu$ better preserves the anchor mapping $\mK_0\mW\approx\mV_0$.
The residual $\mR_t$ measures how well the current edit is satisfied before the $t$-th update.
Thus, increasing $\lambda$ or $\mu$ makes the update more conservative, but may also increase the edit residual when the new edit conflicts with the preservation constraints.

\begin{table*}[ht!]
    \centering
    \small
    \renewcommand{\arraystretch}{1}
    \resizebox{0.65\textwidth}{!}{
        \begin{tabular}{lc|ccccc}
            \toprule[1.5pt]
            {\textbf{Method}} & {\textbf{Model}} & 
            \textbf{Efficacy $\uparrow$} & \textbf{Generalization $\uparrow$} & \textbf{Specificity $\uparrow$} & \textbf{Fluency $\uparrow$} & \textbf{Consistency $\uparrow$} \\
            \midrule
            \midrule

            \textbf{RLSEdit} (Ours)   & \multirow{5}{*}{\rotatebox{90}{\texttt{Llama-3-8B}}}
            & $\mathbf{89.94_{\pm0.75}}$ & $\mathbf{72.84_{\pm1.21}}$ & $\mathbf{60.56_{\pm0.35}}$ & $\mathbf{615.58_{\pm4.34}}$   & $\mathbf{26.27_{\pm0.35}}$ \\
            AlphaEdit & & $66.78_{\pm3.19}$ & $58.27_{\pm1.59}$ & $51.79_{\pm0.70}$ & $\underline{489.91_{\pm33.83}}$  & $\underline{4.59_{\pm0.39}}$ \\
            ROME      & & $47.57_{\pm0.10}$ & $48.45_{\pm0.33}$ & $\underline{52.52_{\pm0.44}}$ & $465.02_{\pm17.88}$  & $1.83_{\pm0.14}$ \\
            MEMIT     & & $49.73_{\pm1.44}$ & $49.24_{\pm0.48}$ & $51.54_{\pm0.68}$ & $323.01_{\pm16.40}$ & $3.45_{\pm1.62}$ \\
            FT        & & $\underline{74.76_{\pm0.00}}$ & $\underline{64.49_{\pm0.00}}$ & $39.69_{\pm0.00}$ & $342.42_{\pm0.20}$   & $1.31_{\pm0.00}$ \\
            \midrule[1pt]

            \textbf{RLSEdit} (Ours)   & \multirow{5}{*}{\rotatebox{90}{\texttt{Qwen2.5-7B}}}
            & $\mathbf{94.45_{\pm1.07}}$ & $\underline{68.55_{\pm0.47}}$ & $\underline{73.37_{\pm0.44}}$ & $\mathbf{625.74_{\pm0.71}}$  & $\underline{31.62_{\pm0.81}}$ \\
            AlphaEdit & & $\underline{94.10_{\pm0.42}}$ & $\mathbf{70.29_{\pm2.30}}$ & $\mathbf{75.29_{\pm0.65}}$ & $\underline{623.51_{\pm0.24}}$  & $31.37_{\pm0.49}$ \\
            ROME      & & $35.70_{\pm1.36}$ & $37.16_{\pm1.19}$ & $65.20_{\pm1.42}$ & $619.67_{\pm16.98}$ & $\mathbf{31.79_{\pm3.59}}$ \\
            MEMIT     & & $53.13_{\pm0.72}$ & $51.39_{\pm0.49}$ & $51.52_{\pm0.92}$ & $532.38_{\pm24.31}$ & $1.63_{\pm2.22}$ \\
            FT        & & $65.72_{\pm0.00}$ & $56.46_{\pm0.00}$ & $45.23_{\pm0.00}$ & $324.70_{\pm0.04}$  & $1.87_{\pm0.03}$ \\
            \bottomrule[1.5pt]
        \end{tabular}}
    \caption{
        \footnotesize CounterFact results on \texttt{Llama-3-8B} and \texttt{Qwen2.5-7B}, comparing \textbf{RLSEdit} with the baselines. 
        We report mean $\pm$ standard deviation over \textbf{3} random seeds, evaluated on the full CounterFact test set after completing all sequential edits (10K edits in total, with a batch size of 100). We evaluate on five metrics: Efficacy, Generalization, Specificity, Fluency, and Consistency. The best results are highlighted in bold, and the second-best results are underlined.
    }
    \label{tab:cf_two_blocks}
\end{table*}

\subsection{Asymptotic Scaling}
\label{subsubsec:asymptotic}

To connect $(\lambda,\mu)$ to the many-edits regime, we view Equation \ref{eq:LS_main} as a ridge-type estimator for a \emph{layer-wise} linear mapping.
We use the statistical model
\begin{equation}
    \label{eq:noise_model_main}
    \mV_i = \mK_i\mW^\star + \mE_i,\qquad \sup_i \E\left\|\mE_i\right\|_F^2<\infty,
\end{equation}
where $\mE_i$ captures approximation error from other layers, context variability, and mismatch in the linearized output.

Sequential edits do not need to be i.i.d.
We only assume long-run stability of the empirical second moments, namely that there exist matrices $\bm{\Sigma}_k$ and $\bm{\Sigma}_{kv}$ such that
\begin{equation}
\label{eq:moment_conv_main}
    \begin{split}
        &\frac{1}{t}\sum_{i=1}^t \mK_i^\top \mK_i \to \bm{\Sigma}_k, \quad
        \frac{1}{t}\sum_{i=1}^t \mK_i^\top \mV_i \to \bm{\Sigma}_{kv},\\
        &\text{s.t.}\;\bm{\Sigma}_k \succ 0\,\text{ on the relevant subspace}.
    \end{split}
\end{equation}
Allow $\lambda=\lambda_t$ and $\mu=\mu_t$ to depend on $t$ and define
\[
\alpha_t \coloneq \lambda_t^2/t,\qquad \beta_t \coloneq \mu_t^2/t.
\]
For asymptotic analysis, it is convenient to work with the normalized objective
$\tilde J_t(\mW) \coloneq J_t(\mW)/t$, which has the same minimizer as $J_t$ for each fixed $t$:
\begin{equation}
    \label{eq:J_normalized_main}
    \begin{split}
         \tilde J_t(\mW) =
        &\frac{1}{t}\sum_{i=1}^t \left\|\mK_i\mW - \mV_i\right\|_F^2 \\
        & + \alpha_t\left\|\mW-\mW_0\right\|_F^2 + \beta_t\left\|\mK_0\mW - \mV_0\right\|_F^2.
    \end{split}
\end{equation}

\begin{proposition}[Asymptotic behavior of the RLS editor]
    \label{prop:rls_asymptotics}
    Assume Equation \ref{eq:noise_model_main}, and the moment convergence in Equation \ref{eq:moment_conv_main}. Also assume
    $\sup_i \E\left\|\mK_i\right\|_F^4\!<\!\infty$ and $\sup_i \E\left\|\mV_i\right\|_F^4\!<\!\infty$.
    Let $\mW_t^*$ be the minimizer of $J_t(\mW)$, with $\alpha_t\!=\!\lambda_t^2/t$ and $\beta_t\!=\!\mu_t^2/t$.
    Suppose that $\alpha_t\!\to\!\alpha$ and $\beta_t\!\to\!\beta$ for some $\alpha,\beta\!\in\![0,\infty)$ as $t\!\to\!\infty$.
    Define the limiting quadratic risk $\mathcal{R}(\mW)$ through the limiting second moments in Equation \ref{eq:moment_conv_main}.
    Then:
    \begin{enumerate}[label=(\roman*),leftmargin=*]
        \item The normalized objectives $\widetilde J_t$ converge pointwise to the regularized limiting risk
            \begin{equation}
                \label{eq:pop_ridge_risk}
                \begin{split}
                    \mathcal{R}_{\mathrm{ridge}}\left(\mW\right) \coloneq &\mathcal{R}\left(\mW\right)+ \alpha \left\|\mW-\mW_0\right\|_F^2 \\
                    &+ \beta\left\|\mK_0 \mW - \mV_0\right\|_F^2.
                \end{split}
            \end{equation}
        \item The function $\mathcal{R}_{\mathrm{ridge}}$ is strictly convex and admits a unique minimizer $\mW^\dagger$.
        \item The RLS editor converges to this minimizer:
            \begin{equation}
                \label{eq:rls_converges_to_Wdagger}
                \mW_t^* \to \mW^\dagger \qquad \text{as } t\to\infty.
            \end{equation}
    \end{enumerate}
\end{proposition}

\begin{remark}
Proposition \ref{prop:rls_asymptotics} is an asymptotic characterization for long-run stable edit streams, not a universal guarantee for arbitrary non-stationary deployments.
It does not assume that edits are i.i.d.; the key requirement is the convergence of the empirical second moments in Equation \ref{eq:moment_conv_main}.
If these limits do not exist, then the editor should not be interpreted as converging to a single limiting risk minimizer.
In that case, the relevant guarantee is the finite-time deviation bound in Proposition \ref{prop:global_deviation}.
\end{remark}

\begin{figure*}[ht]
    \centering
    \includegraphics[width=0.9\textwidth]{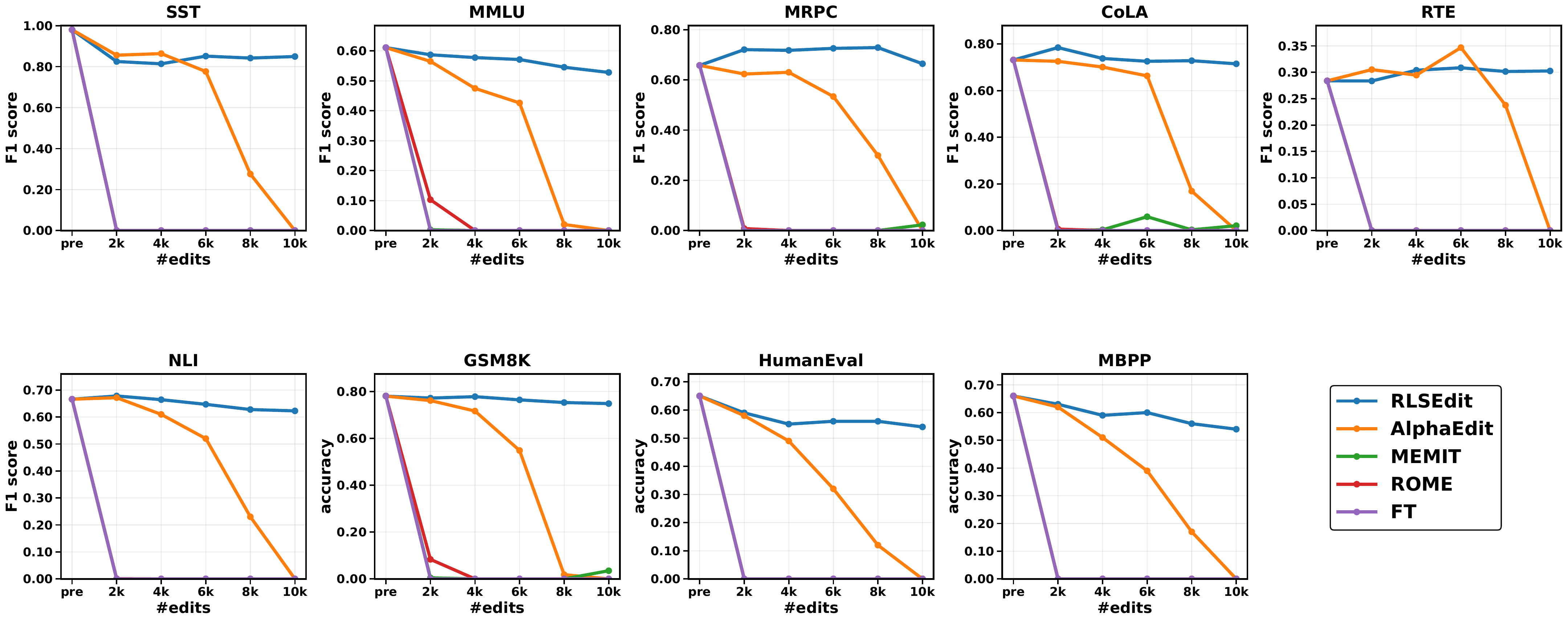}
    \caption{\textbf{General capability preservation.} We evaluate 5 \textsc{GLUE} tasks and additional benchmarks for general knowledge, math reasoning and coding ability (MMLU, GSM8K, HumanEval, MBPP) at multiple editing checkpoints (Pre-edit, 2k--10k edits). \textbf{RLSEdit} is compared against baselines and consistently better preserves the model's general capabilities across tasks and edit scales. The x-axis shows the cumulative number of applied edits, and the y-axis reports the corresponding score (F1 or accuracy).}
    \label{fig:general_results}
\end{figure*}

If $\alpha\!=\!\beta\!=\!0$ (e.g., when $\lambda_t,\mu_t$ are held fixed), then $\mW^\dagger \!=\! \mW^\star$ and $\mW_t^*$ converges to the least-squares limiting minimizer.
If $\alpha\!>\!0$ and/or $\beta\!>\!0$, then $\mW^\dagger$ interpolates between the data-driven optimum $\mW^\star$ and the anchor constraints encoded by $(\mW_0,\mK_0,\mV_0)$:
larger $\alpha$ shrinks $\mW^\dagger$ toward $\mW_0$, and larger $\beta$ enforces $\mK_0\mW^\dagger\!\approx\! \mV_0$ even as $t\!\to\!\infty$.
Thus, increasing $\lambda$ and $\mu$ makes the update more conservative.
This improves preservation, but can reduce edit adherence when the new edit conflicts with the original mapping.
A common policy is to set a deviation budget for the associated penalties, then tune $(\lambda,\mu)$ to satisfy both bounds.

The limits $\mu,\lambda\to\infty$ yield hard anchoring ($\mK_0\mW\!=\!\mV_0$) and frozen parameters ($\mW_t^*\!\to\!\mW_0$).

%% file: sections/5_experiments.tex
\section{Experiments and Results}
\label{sec:experiments}

\subsection{Experimental Setup}
\label{subsec:setup}

\paragraph{Models and Baselines.}
We conduct experiments with two backbone models, \texttt{Llama3-8B} and \texttt{Qwen2.5-7B}, against \textsc{AlphaEdit}~\citep{fang2410alphaedit}, \textsc{ROME}~\citep{meng2022locating}, \textsc{MEMIT}~\citep{meng2022mass}, and fine-tuning (\textsc{FT})~\citep{Zhu2020ModifyingMI}.

\paragraph{Datasets and Metrics.}
Following prior work, we mainly use the \textbf{CounterFact} dataset~\citep{meng2022locating}. We report \textbf{Efficacy} (rewrite success), \textbf{Generalization} (paraphrase success), \textbf{Specificity} (neighborhood success), \textbf{Fluency} (generation entropy), and \textbf{Consistency} (reference score). To further test whether the trend holds beyond CounterFact, we also evaluate on \textbf{ZsRE}.

\begin{figure} 
    \centering
    \includegraphics[width=\linewidth]{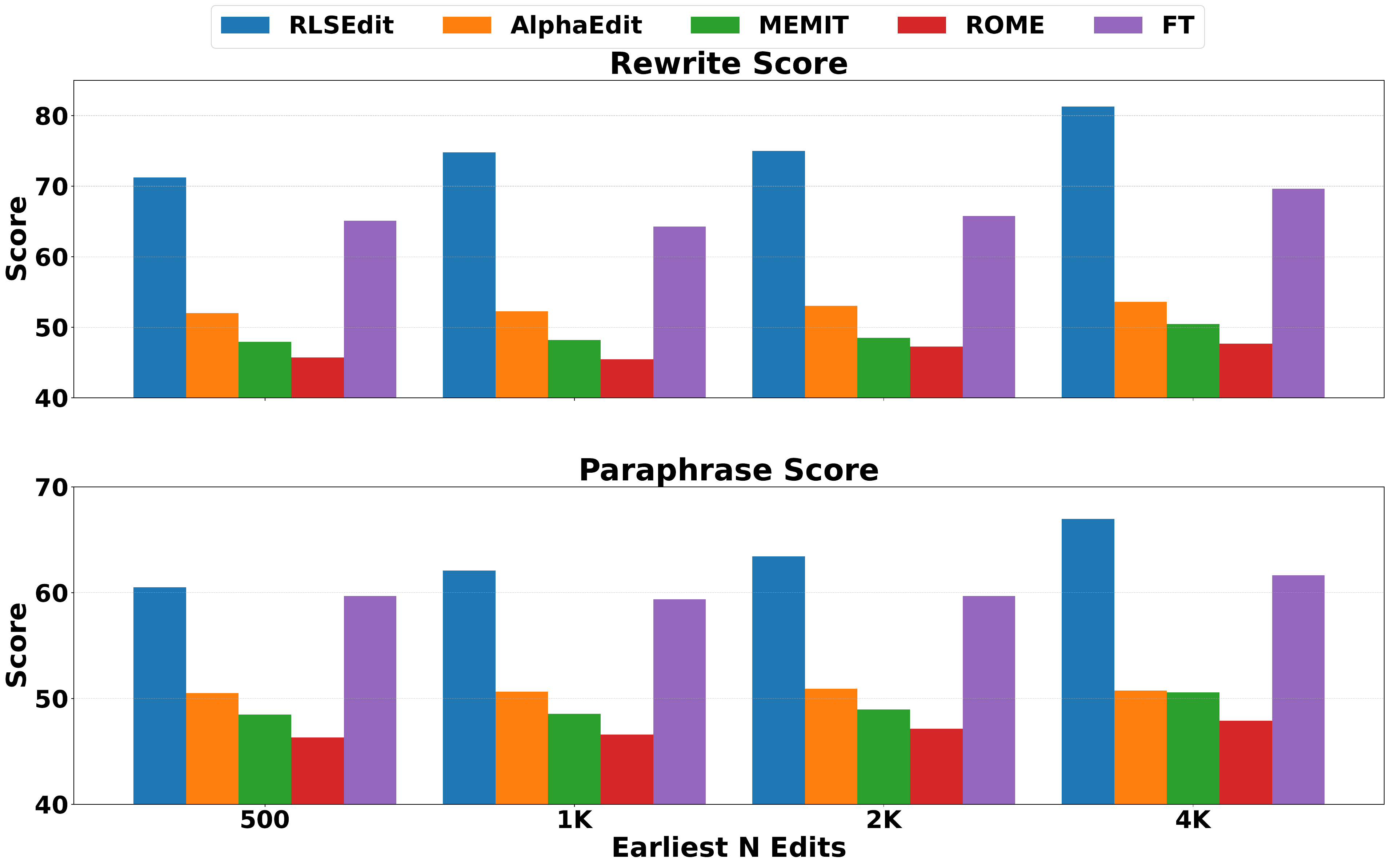}
    \caption{Improvements on early edits. After applying 10K sequential edits, we re-evaluate performance on the earliest edited cases (500, 1K, 2K, 4K). Each bar reports the Rewrite or Paraphrase score. \textbf{RLSEdit} consistently achieves the highest scores across all settings.}
    \label{fig:early_edits}
\end{figure}

\subsection{Main Results}
\label{subsec:results}

\paragraph{Editing Results.} Table 
\ref{tab:cf_two_blocks} reports performance after $10\mathsf{K}$ edits with batch size 100 on CounterFact. On \texttt{Llama3-8B}, \textbf{RLSEdit} achieves the best scores across all five metrics. In particular, it obtains a clear improvement in Efficacy over the second-best method (89.94 vs. 74.76), while also improving Generalization, Specificity, Fluency, and Consistency.

On \texttt{Qwen2.5-7B}, the results are more nuanced. \textbf{RLSEdit} achieves the best Efficacy and Fluency, while \textsc{AlphaEdit} gives slightly higher Generalization and Specificity, and \textsc{ROME} gives the best Consistency. This difference suggests that the editing--preservation trade-off is model-dependent rather than evidence of instability in \textbf{RLSEdit}. Overall, \textbf{RLSEdit} remains competitive with \textsc{AlphaEdit} on \texttt{Qwen2.5-7B} and substantially outperforms \textsc{ROME}, \textsc{MEMIT}, and \textsc{FT} on most metrics.

\begin{table}[t]
    \centering
    \scriptsize
    \setlength{\tabcolsep}{4pt}
    \renewcommand{\arraystretch}{0.95}
    
    \resizebox{0.9\columnwidth}{!}{
    \begin{tabular}{lcccc}
        \toprule[1.1pt]
        \textbf{Metric} & \textbf{ROME} & \textbf{MEMIT} & \textbf{AlphaEdit} & \textbf{RLSEdit} \\
        \midrule
        ZsRE Efficacy $\uparrow$ & 1.52 & 0.96 & \underline{70.04} & \textbf{81.20} \\
        \bottomrule[1.1pt]
    \end{tabular}
    }

    \caption{
        \footnotesize
        ZsRE results on \texttt{Llama3-8B}. \textbf{RLSEdit} shows the same trend beyond CounterFact.
    }
    \label{tab:zsre_results_horizontal}
\end{table}

\paragraph{ZsRE Results.}
We further test \textbf{RLSEdit} on ZsRE \citep{levy-etal-2017-zero}. As shown in Table \ref{tab:zsre_results_horizontal}, \textbf{RLSEdit} achieves 81.20 Efficacy on \texttt{Llama3-8B}, outperforming \textsc{AlphaEdit}, \textsc{MEMIT}, and \textsc{ROME}. This supports the generality of the empirical trend across different editing datasets.
\paragraph{General Capability Results.}  
To assess how well editing methods preserve the pre-edited model's general abilities, we evaluate five tasks from \textsc{GLUE}~\citep{wang-etal-2018-glue}, together with additional benchmarks for \emph{general knowledge} (MMLU), \emph{math reasoning} (GSM8K), and \emph{coding ability} (HumanEval, MBPP). We conduct the evaluation on multiple editing checkpoints of the \texttt{Llama3-8B} model, using 10\textsf{K} total edits with a batch size of 100.

Figure
\ref{fig:general_results} summarizes the general capability evaluations. Across the evaluated language understanding, knowledge, math, and coding benchmarks, \textbf{RLSEdit} better preserves the model's general capabilities throughout the editing trajectory. This is especially important in the long sequential setting, where many edits may accumulate interference over time. In contrast, \textsc{MEMIT}, \textsc{ROME}, and \textsc{FT} show stronger degradation as edits accumulate. \textsc{AlphaEdit} performs competitively in the early stages, but drops more noticeably after a large number of edits. These results suggest that \textbf{RLSEdit} provides a better balance between edit reliability and general capability preservation.

\subsection{Analysis and Discussion}

\begin{table}[ht]
    \centering
    \small
    \setlength{\tabcolsep}{4pt}
    \renewcommand{\arraystretch}{1.05}
    \begin{tabular}{lcccccc}
        \toprule[1.2pt]
        \multirow{2}{*}{\centering\textbf{Method}} &
        \multicolumn{3}{c}{\textbf{\texttt{Llama3-8B}}} &
        \multicolumn{3}{c}{\textbf{\texttt{Qwen2.5-7B}}} \\
        \cmidrule(lr){2-4}\cmidrule(lr){5-7}
         & 100 & 200 & 500 & 100 & 200 & 500 \\
        \midrule
        AlphaEdit & 525.15 & 227.93 & 108.07 & 978.32 & 412.94 & 197.49 \\
        \textbf{RLSEdit} & \textbf{328.39} & \textbf{166.84} & \textbf{66.85} & \textbf{545.65} & \textbf{271.20} & \textbf{112.88} \\
        \bottomrule[1.2pt]
    \end{tabular}
    \caption{
        Update time (seconds) for performing 10K edits on \texttt{Llama3-8B} and \texttt{Qwen2.5-7B} using batch sizes \{100, 200, 500\}. Lower values indicate faster updates. We compare \textbf{RLSEdit} with \textsc{AlphaEdit}.
    }
    \label{tab:efficiency_update_time}
\end{table}

\paragraph{Early Edits Comparison.}
To examine how well \textbf{RLSEdit} and the baselines preserve earlier edits in a sequential editing setting, we re-evaluate the first \(N\) edited cases, where \(N \in \{500, 1\mathsf{K}, 2\mathsf{K}, 4\mathsf{K}\}\), after performing \(10\mathsf{K}\) sequential edits with batch size 100 on \texttt{Llama3-8B}. As shown in Figure \ref{fig:early_edits}, \textbf{RLSEdit} consistently achieves the best retention across all \(N\): its Rewrite scores range from \(71.22\) at \(N{=}500\) to \(81.28\) at \(N{=}4\mathsf{K}\), and its Paraphrase scores range from \(60.49\) to \(66.98\). In contrast, the baselines remain noticeably lower, suggesting weaker preservation of previously edited knowledge under long sequential editing.

\paragraph{Effect of Regularization.}
The two regularization parameters play different roles. The parameter $\lambda$ controls the drift from the original weights $\mW_0$, while $\mu$ controls the anchor mapping $\mK_0\mW$. We find that the effect of $\lambda$ is model-dependent. On \texttt{Llama3-8B}, setting $\lambda=0$ while keeping $\mu=12000$ reduces the Efficacy score to 87.0, with the other metrics also slightly lower. This suggests that the parameter-drift penalty improves the editing--preservation trade-off for this backbone. In contrast, \texttt{Qwen2.5-7B} performs best with $\lambda=0$, indicating that this model is more robust to direct weight changes.

The anchor parameter $\mu$ is important in practice. Removing it can make the linear system ill-conditioned in long sequential editing. For a new backbone, we therefore tune $(\lambda,\mu)$ using a small pilot sweep and monitor the three objective terms: edit fitting, parameter drift, and anchor-mapping error.

\paragraph{Speed-up Analysis.}
Table \ref{tab:efficiency_update_time} reports the update computation time for \textbf{RLSEdit} and \textsc{AlphaEdit} across two model backbones and three batch sizes. Across all six configurations, \textbf{RLSEdit} runs faster, reducing update time by $1.37\times$--$1.79\times$ relative to \textsc{AlphaEdit}. This is consistent with the theoretical time-complexity analysis in Section \ref{subsec:complexity}.

%% file: sections/6_conclusion.tex
\section{Conclusion}
Existing model editing methods often degrade when the number of edits becomes large. We propose \textbf{RLSEdit}, a recursive least-squares framework for long sequential editing with soft editing and soft preservation constraints. RLSEdit introduces two regularization terms to control deviation from the pre-trained weights and from an anchor mapping, allowing it to balance edit adherence and preservation. The resulting objective admits an efficient Woodbury-based recursion whose per-edit update cost is independent of the edit history length.

Empirically, RLSEdit scales stably to $10\mathsf{K}$ edits on \texttt{Llama-3} and \texttt{Qwen2.5}, achieves strong results on CounterFact and ZsRE, and better preserves general capabilities across language understanding, knowledge, math, and coding benchmarks. These results support RLSEdit as a practical approach for continuous model editing. A remaining limitation is that our experiments focus on 7B--8B backbones; evaluating larger models and more failure cases is an important direction for future work.

%% file: sections/x_appendix.tex
\section{Preliminaries.}
Recall the stacked least-squares form
\begin{align}
\mA_t=\begin{bmatrix}\lambda I\\[1pt]\mu \mK_0\\[1pt]\mK_1\\\vdots\\\mK_t\end{bmatrix},
\mB_t=\begin{bmatrix}\lambda \mW_0\\[1pt]\mu \mV_0\\[1pt]\mV_1\\\vdots\\\mV_t\end{bmatrix},
\\
\mW_t^*=\arg\min_\mW \|\mA_t\mW-\mB_t\|_F^2, \label{eq:app-prelim-stack}
\end{align}
and define the normal-equation matrices
\begin{align}
\mS_t&\coloneq\mA_t^\top \mA_t
=\lambda^2 I+\mu^2 \mK_0^\top \mK_0+\sum_{i=1}^t \mK_i^\top \mK_i, \label{eq:app-prelim-St}\\
\mT_t&\coloneq\mA_t^\top \mB_t
=\lambda^2 \mW_0+\mu^2 \mK_0^\top \mV_0+\sum_{i=1}^t \mK_i^\top \mV_i. \label{eq:app-prelim-Tt}
\end{align}
Whenever $\mS_t\succ0$, the minimizer is unique and satisfies
\begin{align}
\mW_t^*&=\mS_t^{-1}\mT_t,
&
\mC_t&\coloneq\mS_t^{-1}. \label{eq:app-closed}
\end{align}
We will use the one-step identity (a standard RLS consequence of stacking and normal equations)
\begin{align}
\mW_t^*-\mW_{t-1}^*
&= \mC_t\,\mK_t^\top\,\mR_t,
&
\mR_t&\coloneq\mV_t-\mK_t\mW_{t-1}^*. \label{eq:app-rls-step}
\end{align}
Finally, we assume the anchor is satisfied by the initializer:
\begin{align}
\mK_0\mW_0=\mV_0. \label{eq:app-anchor}
\end{align}

\subsubsection*{Alternative: streaming QR update}
\label{subsubsec:qr}

For improved numerical stability, one may maintain a QR factorization of $\mA_t$.
Assume that at time $t-1$ we have orthogonal transforms
\begin{equation}
\label{eq:OT_main}
\mQ_{t-1}^\top \mA_{t-1} = \begin{bmatrix} \mR_{t-1}\\\bm{0}\end{bmatrix},
\qquad
\mQ_{t-1}^\top \mB_{t-1} = \begin{bmatrix} \bar \mB_{t-1}\\ \tilde \mB_{t-1}\end{bmatrix},
\end{equation}
where $\mR_{t-1}\in\mathbb{R}^{d_K\times d_K}$ is upper triangular.
At time $t$, we apply additional orthogonal transforms $\bar \mQ_t$ to
\begin{equation}
\label{eq:QR_stream_main}
\bar \mQ_t^\top
\begin{bmatrix} \mR_{t-1}\\ \mK_t\end{bmatrix}
=
\begin{bmatrix} \mR_t\\ \bm{0}\end{bmatrix},
\qquad
\bar \mQ_t^\top
\begin{bmatrix} \bar \mB_{t-1}\\ \mV_t\end{bmatrix}
=
\begin{bmatrix} \bar \mB_t\\ \hat \mB_t\end{bmatrix}.
\end{equation}
Then $\mW_t^*$ is obtained by solving the triangular system
\begin{equation}
\label{eq:tri_main}
\mR_t\mW_t^* = \bar\mB_t.
\end{equation}

\paragraph{Initialization.}
Since $\mR_0^\top \mR_0 = \lambda^2\mI + \mu^2 \mK_0^\top \mK_0$,
we compute $\mR_0$ via Cholesky and set $\bar \mB_0 = \mR_0\mW_0$
(using $\mK_0\mW_0=\mV_0$).

\subsection*{A. Proof of Theorem~\ref{prop:global_deviation}}

\begin{proof}[Proof of Theorem~\ref{prop:global_deviation}(i) (parameter deviation)]
The normal equations for $\mW_T^*$ are
\begin{align}
\mS_T \mW_T^*&=\mT_T, \label{eq:app-ne-T}
\end{align}
where
\begin{align}
\mS_T
&=\lambda^2 \mI+\mu^2 \mK_0^\top \mK_0+\sum_{i=1}^T \mK_i^\top \mK_i,
\label{eq:app-ne-ST}\\
\mT_T
&=\lambda^2 \mW_0+\mu^2 \mK_0^\top V_0+\sum_{i=1}^T \mK_i^\top \mV_i.
\label{eq:app-ne-TT}
\end{align}
Using the anchor condition \eqref{eq:app-anchor}, we have
\begin{align}
\mK_0^\top \mV_0
&=\mK_0^\top \mK_0\mW_0. \label{eq:app-K0tV0}
\end{align}
Subtracting $\mS_T\mW_0$ from both sides of \eqref{eq:app-ne-T} gives
\begin{align}
\mS_T(\mW_T^*-\mW_0)
&=
\mT_T-\mS_T\mW_0
=\sum_{i=1}^T \mK_i^\top(\mV_i-\mK_i\mW_0).
\label{eq:app-param-core}
\end{align}
Thus
\begin{align}
\mW_T^*-\mW_0
&=\mS_T^{-1}\sum_{i=1}^T \mK_i^\top(\mV_i-\mK_i\mW_0).
\label{eq:app-param-solved}
\end{align}
By submultiplicativity and $\mS_T\succeq \lambda^2 I$ (when $\lambda>0$),
\begin{align}
\|\mW_T^*-\mW_0\|_F
&\le
\|\mS_T^{-1}\|_2 \left\|\sum_{i=1}^T \mK_i^\top(\mV_i-\mK_i\mW_0)\right\|_F \nonumber\\
&\le
\frac{1}{\lambda^2}
\left\|\sum_{i=1}^T \mK_i^\top(\mV_i-\mK_i\mW_0)\right\|_F,
\label{eq:app-param-final}
\end{align}
which proves the claim.
\end{proof}

\begin{lemma}\label{lem:app-one-step}
Assume $\mu>0$ and $\mS_t\succ0$. Then for each $t\ge1$,
\begin{align}
\|\mK_0(\mW_t^*-\mW_{t-1}^*)\|_F
&\le \frac{1}{\mu}\,\|\mR_t\|_F, \label{eq:app-tight-step}\\
\|\mK_0(\mW_t^*-\mW_{t-1}^*)\|_F
&\le \|\mK_0\|_2\,\|\mK_t\|_2\,\|\mC_t\|_2\,\|\mR_t\|_F. \label{eq:app-classic-step}
\end{align}
Moreover,
\begin{align}
\|\mC_t\|_2
\le
\frac{1}{\lambda^2+\mu^2\bm{\Sigma}_{\min}^2(\mK_0)+\sum_{i=1}^{t}\bm{\Sigma}_{\min}^2(\mK_i)}.
\label{eq:app-Ct-floor}
\end{align}
\end{lemma}

\begin{proof}
From \eqref{eq:app-rls-step},
\begin{align}
\mK_0(\mW_t^*-\mW_{t-1}^*)
&=\mK_0\mC_t\mK_t^\top \mR_t. \label{eq:app-step-K0}
\end{align}
The classical bound \eqref{eq:app-classic-step} follows from operator norm submultiplicativity:
\begin{align}
\|\mK_0\mC_t\mK_t^\top \mR_t\|_F
&\le
\|\mK_0\|_2\,\|\mC_t\|_2\,\|\mK_t\|_2\,\|\mR_t\|_F.
\label{eq:app-classic-step-proof}
\end{align}

For the tighter bound \eqref{eq:app-tight-step}, consider
\begin{align}
\|\mK_0\mC_t\mK_t^\top \mR_t\|_F^2
&=
\mathrm{tr}\!\Big(
R_t^\top \mK_t \mC_t \mK_0^\top \mK_0 \mC_t \mK_t^\top \mR_t
\Big) \nonumber\\
&\le
\big\|\mK_t \mC_t \mK_0^\top \mK_0 \mC_t \mK_t^\top\big\|_2\,\|\mR_t\|_F^2.
\label{eq:app-tight-qf}
\end{align}
Using $\|\mM\mN\mM^\top\|_2\le \|\mM\|_2^2\|\mN\|_2$ with $\mM=\mK_t\mC_t^{1/2}$ and
$\mN=\mC_t^{1/2}\mK_0^\top \mK_0 \mC_t^{1/2}$,
\begin{align}
\big\|\mK_t \mC_t \mK_0^\top \mK_0 \mC_t \mK_t^\top\big\|_2
&\le
\|\mK_t \mC_t \mK_t^\top\|_2\cdot \|\mK_0 \mC_t \mK_0^\top\|_2.
\label{eq:app-factorize}
\end{align}
We bound the two factors.

\emph{(a) $\|\mK_t \mC_t \mK_t^\top\|_2\le 1$.}
Let $\mC_{t-1}\coloneq\mS_{t-1}^{-1}$ and define
\begin{align}
\mH_t\coloneq\mK_t \mC_{t-1}K_t^\top \succeq 0. \label{eq:app-Ht}
\end{align}
By Sherman--Morrison--Woodbury,
\begin{align}
\mC_t
&=
\mC_{t-1}-\mC_{t-1}\mK_t^\top(I+\mH_t)^{-1}\mK_t \mC_{t-1}.
\label{eq:app-smw}
\end{align}
Hence,
\begin{align}
\mK_t \mC_t \mK_t^\top
&=
\mH_t - \mH_t(I+\mH_t)^{-1}\mH_t
=
\mH_t(I+\mH_t)^{-1}.
\label{eq:app-htfrac}
\end{align}
The eigenvalues of $\mH_t(I+\mH_t)^{-1}$ are $h/(1+h)\in[0,1)$ for $h\ge0$, so
$\|\mK_t \mC_t \mK_t^\top\|_2\le 1$.

\emph{(b) $\|\mK_0 \mC_t \mK_0^\top\|_2\le 1/\mu^2$.}
Since $\mS_t\succeq \mu^2 \mK_0^\top \mK_0$, we have $\mC_t=\mS_t^{-1}\preceq (\mu^2 \mK_0^\top \mK_0)^{\dagger}$ on the support of $\mK_0^\top \mK_0$, hence
\begin{align}
\mK_0 \mC_t \mK_0^\top
&\preceq
\frac{1}{\mu^2}\,\mK_0 (\mK_0^\top \mK_0)^{\dagger}\mK_0^\top
=
\frac{1}{\mu^2}\mP_{\mK_0},
\label{eq:app-proj}
\end{align}
where $P_{K_0}$ is the orthogonal projector onto $\mathrm{Row}(\mK_0)$.
Therefore $\|\mK_0 \mC_t \mK_0^\top\|_2\le 1/\mu^2$.

Combining \eqref{eq:app-tight-qf}--\eqref{eq:app-proj} yields
\begin{align}
\|\mK_0\mC_t\mK_t^\top \mR_t\|_F^2
\le \frac{1}{\mu^2}\|\mR_t\|_F^2,
\end{align}
which implies \eqref{eq:app-tight-step}.

Finally, for \eqref{eq:app-Ct-floor}, note that
\begin{align}
\mS_t
&=\lambda^2 \mI+\mu^2 \mK_0^\top \mK_0+\sum_{i=1}^t \mK_i^\top \mK_i \nonumber\\
&\succeq
 \left(\lambda^2+\mu^2\bm{\Sigma}_{\min}^2(\mK_0)+\sum_{i=1}^t \bm{\Sigma}_{\min}^2(\mK_i) \right)\mI,
\end{align}
hence $\|\mC_t\|_2=1/\lambda_{\min}(\mS_t)$ implies \eqref{eq:app-Ct-floor}.
\end{proof}

\begin{proof}[Proof of Theorem~\ref{prop:global_deviation}(ii) and the adaptive spectral variant]
Telescoping gives
\begin{align}
\mW_T^*-\mW_0
&=\sum_{t=1}^T (\mW_t^*-\mW_{t-1}^*). \label{eq:app-telescope}
\end{align}
Left-multiply by $\mK_0$ and apply the triangle inequality:
\begin{align}
\|\mK_0(\mW_T^*-\mW_0)\|_F
&\le
\sum_{t=1}^T \|\mK_0(\mW_t^*-\mW_{t-1}^*)\|_F.
\label{eq:app-tri}
\end{align}
Applying Lemma~\ref{lem:app-one-step} with \eqref{eq:app-tight-step} termwise yields
\begin{align}
\|\mK_0(\mW_T^*-\mW_0)\|_F
&\le
\frac{1}{\mu}\sum_{t=1}^T \|\mR_t\|_F,
\label{eq:app-anchor-global}
\end{align}
which proves Theorem~\ref{prop:global_deviation}(ii).

For the adaptive spectral variant, apply instead \eqref{eq:app-classic-step} and \eqref{eq:app-Ct-floor}:
\begin{align}
& \|\mK_0(\mW_t^*-\mW_{t-1}^*)\|_F\\ 
\le & \|\mK_0\|_2\,\|\mK_t\|_2\,\|\mC_t\|_2\,\|\mR_t\|_F \nonumber\\
\le & \frac{\|\mK_0\|_2\,\|\mK_t\|_2}{ \lambda^2+\mu^2\bm{\Sigma}_{\min}^2(\mK_0)+\sum_{i=1}^t \bm{\Sigma}_{\min}^2(\mK_i) }\,\|\mR_t\|_F.
\label{eq:app-adaptive-step}
\end{align}
Summing \eqref{eq:app-adaptive-step} over $t=1,\dots,T$ gives the stated inequality.
\end{proof}

\section*{B: Proofs for Proposition~\ref{prop:rls_asymptotics}}
\paragraph{Step 1: Expand the normalized objective.}
Let $\widetilde J_t(W)$ denote the normalized objective
\[
\widetilde J_t(W)
=
\frac{1}{t}\sum_{i=1}^t \|K_iW-V_i\|_F^2
+\alpha_t\|W-W_0\|_F^2
+\beta_t\|K_0W-V_0\|_F^2 .
\]
Expand the data-fit term using $\|K_iW-V_i\|_F^2
=\operatorname{tr}(W^\top K_i^\top K_i W)-2\operatorname{tr}(W^\top K_i^\top V_i)+\|V_i\|_F^2$ to obtain
\begin{equation}
\label{eq:Jt_expand}
\begin{aligned}
\widetilde J_t(W)
&=
\operatorname{tr}\!\big(W^\top \widehat\Sigma_{K,t} W\big)
-2\,\operatorname{tr}\!\big(W^\top \widehat\Sigma_{KV,t}\big)
+c_t \\
&\quad
+\alpha_t\|W-W_0\|_F^2
+\beta_t\|K_0W-V_0\|_F^2 .
\end{aligned}
\end{equation}

\paragraph{Step 2: Proof of (i) (pointwise convergence).}
Fix any $W$.
By the assumed moment convergence \ref{eq:moment_conv_main},
\[
\widehat\Sigma_{K,t}\to \Sigma_K,
\qquad
\widehat\Sigma_{KV,t}\to \Sigma_{KV}.
\]
Moreover, by the bounded fourth-moment assumption
$\sup_i \E\|V_i\|_F^4<\infty$, we have $\sup_i \E\|V_i\|_F^2<\infty$,
so $\{c_t\}$ is tight and (along the same probability-1 event used for the
empirical-moment convergence) converges to the constant $\E\|V\|_F^2$. Finally, $\alpha_t\to\alpha$ and $\beta_t\to\beta$ by assumption.
Taking limits in \ref{eq:Jt_expand} yields, for each fixed $W$,
\begin{equation}
    \begin{split}
        \widetilde J_t(W)
        \;\longrightarrow\;
        &\operatorname{tr}\!\big(W^\top \Sigma_K W\big)
        -2\,\operatorname{tr}\!\big(W^\top \Sigma_{KV}\big)
        +\E\|V\|_F^2 \\
        &+\alpha\|W-W_0\|_F^2
        +\beta\|K_0W-V_0\|_F^2.
    \end{split}
\end{equation}

The right-hand side equals $\mathcal R(W)+\alpha\|W-W_0\|_F^2+\beta\|K_0W-V_0\|_F^2$,
i.e., $\mathcal R_{\mathrm{ridge}}(W)$ up to an additive constant. This proves (i).

\paragraph{Step 3: Proof of (ii) (strict convexity and uniqueness).}
\begin{equation}
    \begin{split}
        \mathcal R_{\mathrm{ridge}}(W)
        =
        &\operatorname{tr}\!\big(W^\top \Sigma_K W\big)
        -2\,\operatorname{tr}\!\big(W^\top \Sigma_{KV}\big) \\
        &+\alpha\|W-W_0\|_F^2
        +\beta\|K_0W-V_0\|_F^2
        +\text{const}
    \end{split}
\end{equation}
Its Hessian (with respect to $W$) is the linear operator
\begin{equation}
\label{eq:ridge_hessian}
\nabla^2 \mathcal R_{\mathrm{ridge}}(W)
\;=\;
2\Big(\Sigma_K+\alpha I+\beta K_0^\top K_0\Big),
\end{equation}
acting identically on each of the $d_V$ columns.
Under the assumption in \ref{eq:moment_conv_main} that $\Sigma_K\succ 0$
(on the relevant subspace), and since $\alpha,\beta\ge 0$, the matrix
$\Sigma_K+\alpha I+\beta K_0^\top K_0$ is positive definite on that subspace.
Hence $\mathcal R_{\mathrm{ridge}}$ is strictly convex and admits a unique minimizer $W^\dagger$.
This proves (ii).

\paragraph{Step 4: Proof of (iii) (consistency via closed form).}
Because $\widetilde J_t$ is quadratic, its minimizer $W_t^*$ has the closed form
\begin{equation}
\label{eq:W_closed_empirical}
W_t^*
=
\Big(\widehat\Sigma_{K,t}+\alpha_t I+\beta_t K_0^\top K_0\Big)^{-1}
\Big(\widehat\Sigma_{KV,t}+\alpha_t W_0+\beta_t K_0^\top V_0\Big).
\end{equation}
Similarly, the unique minimizer $W^\dagger$ of $\mathcal R_{\mathrm{ridge}}$ satisfies
\begin{equation}
\label{eq:W_closed_population}
W^\dagger
=
\Big(\Sigma_K+\alpha I+\beta K_0^\top K_0\Big)^{-1}
\Big(\Sigma_{KV}+\alpha W_0+\beta K_0^\top V_0\Big).
\end{equation}

By \ref{eq:moment_conv_main} and $\alpha_t\to\alpha$, $\beta_t\to\beta$,
the matrices and right-hand sides in \ref{eq:W_closed_empirical} converge:
\[
\widehat\Sigma_{K,t}+\alpha_t I+\beta_t K_0^\top K_0
\;\longrightarrow\;
\Sigma_K+\alpha I+\beta K_0^\top K_0,
\]
\[
\widehat\Sigma_{KV,t}+\alpha_t W_0+\beta_t K_0^\top V_0
\;\longrightarrow\;
\Sigma_{KV}+\alpha W_0+\beta K_0^\top V_0.
\]
By (ii), the limit matrix $\Sigma_K+\alpha I+\beta K_0^\top K_0$ is invertible
(on the relevant subspace), and matrix inversion is continuous on the set of
invertible matrices. Therefore, taking limits in \ref{eq:W_closed_empirical} yields
\[
W_t^*
\;\longrightarrow\;
W^\dagger,
\]
along the same probability-1 event, which establishes almost sure convergence.
This proves (iii) and completes the proof.

\subsection*{C. Useful limit regimes (hard constraints as limits)}

\begin{corollary}[Hard limits from soft penalties]\label{cor:app-limits}
Fix $T$ and $\{(\mK_i,\mV_i)\}_{i=0}^T$, and assume the anchor condition \eqref{eq:app-anchor}.
Let $\mW_T^*$ minimize \eqref{eq:LS_main} at time $T$ and define
\begin{align}
\mD_T&\coloneq\|\mK_0(\mW_T^*-\mW_0)\|_F,
&
\mP_T&\coloneq\|\mW_T^*-\mW_0\|_F.
\end{align}
Then:
\begin{enumerate}[label=(\roman*),leftmargin=*]
\item (\emph{Hard anchor as $\mu\to\infty$.}) For any fixed $\lambda\ge0$,
\begin{align}
\mD_T
\le
\frac{1}{\mu}
\left(\sum_{i=1}^T \|\mK_i\mW_0-\mV_i\|_F^2\right)^{1/2}
\label{eq:app-mu-limit}
\end{align}
hence as \(\mu\to\infty, \mD_T\to 0\).
\item (\emph{Freezing as $\lambda\to\infty$.}) For any fixed $\mu\ge0$,
\begin{align}
\mP_T
\le
\frac{1}{\lambda}
\left(\sum_{i=1}^T \|\mK_i\mW_0-\mV_i\|_F^2\right)^{1/2},
\,
\label{eq:app-lam-limit}
\end{align}
hence as \(\lambda\to\infty, \mP_T\to 0\), and consequently $\mD_T\to0$ as well.
\end{enumerate}
\end{corollary}

\begin{proof}
Let $\Phi_T(\mW)$ denote the objective \eqref{eq:LS_main} at time $T$.
Since $\mW_T^*$ is the minimizer, $\Phi_T(\mW_T^*)\le \Phi_T(\mW_0)$.
Using $\mK_0\mW_0=\mV_0$, we have
\begin{align}
\Phi_T(\mW_0)=\sum_{i=1}^T \|\mK_i\mW_0-\mV_i\|_F^2.
\label{eq:app-phiW0}
\end{align}

\emph{(i) $\mu\to\infty$.}
From $\Phi_T(\mW_T^*)\le \Phi_T(\mW_0)$,
\begin{align}
\mu^2\|\mK_0\mW_T^*-\mV_0\|_F^2
\le
\Phi_T(\mW_T^*)
\le
\Phi_T(\mW_0).
\end{align}
Since $\mD_T=\|\mK_0(\mW_T^*-\mW_0)\|_F=\|\mK_0\mW_T^*-\mV_0\|_F$, combining with \eqref{eq:app-phiW0} yields \eqref{eq:app-mu-limit}.

\emph{(ii) $\lambda\to\infty$.}
Similarly,
\begin{align}
\lambda^2\|\mW_T^*-\mW_0\|_F^2
\le
\Phi_T(\mW_T^*)
\le
\Phi_T(\mW_0),
\end{align}
and \eqref{eq:app-phiW0} implies \eqref{eq:app-lam-limit}. Then $\mD_T\le \|\mK_0\|_2 \mP_T\to0$.
\end{proof}

\section{Detailed Hyperparameter Settings}\label{sec:hyperparameter}

For sequential editing experiments, we perform 10\textsf{K} edits for all methods. For methods that support batch editing (MEMIT, AlphaEdit, and RLSEdit), we use a batch size of 100. We edit layers \{4,5,6,7,8\} for \texttt{Llama3-8B} and layers \{7,8,9,10,11\} for \texttt{Qwen2.5-7B} for these methods. For \textsc{ROME}, we edit a single layer, using layer 5 for \texttt{Llama3-8B} and layer 11 for \texttt{Qwen2.5-7B}. For RLSEdit regularization, on \texttt{Llama3-8B}, we set $\lambda=3$ and $\mu=20000$ and on \texttt{Qwen2.5-7B}, we set $\lambda=0$ and $\mu=12000$.

\section{General Capability Benchmarks}
\label{sec:benchmarks}
Here we list the benchmarks used in general capability tests (5 GLUE experiments, MMLU, GSM8K, HumanEval, and MBPP).

\textbf{GLUE Tasks} involve:
\begin{itemize}
    \item \textsc{SST (Stanford Sentiment Treebank)}~\citep{socher-etal-2013-recursive}: A sentence-level sentiment classification task that predicts the sentiment polarity of a given sentence.

    \item \textsc{MRPC (Microsoft Research Paraphrase Corpus)}~\citep{dolan-brockett-2005-automatically}: A sentence-pair task that determines whether two sentences are paraphrases of each other. 
    \item \textsc{CoLA (Corpus of Linguistic Acceptability)}~\citep{warstadt-etal-2019-neural}: A grammatical acceptability task that predicts whether a sentence is linguistically acceptable. 
    \item \textsc{RTE (Recognizing Textual Entailment)}~\citep{DBLP:conf/tac/BentivogliMDDG09}: A natural language inference (NLI) task in a binary setting. Given a premise and a hypothesis, the model predicts whether the premise entails the hypothesis. 
    \item \textsc{NLI (Natural Language Inference; commonly MNLI-style)}~\citep{williams-etal-2018-broad}: A sentence-pair inference task that predicts the semantic relation between a premise and a hypothesis. 
\end{itemize}

\textsc{MMLU (Massive Multi-task Language Understanding)}~\citep{Hendrycks2020MeasuringMM}: A task that measures broad factual knowledge and reasoning.

\textsc{GSM8K (Grade School Math 8K)}~\citep{Cobbe2021TrainingVT}: A math word-problem dataset that evaluates step-by-step arithmetic reasoning.

\textsc{HumanEval}~\citep{Chen2021EvaluatingLL}: A code generation benchmark where models synthesize Python functions from natural-language problem descriptions and are evaluated by unit tests.

\textsc{MBPP (Mostly Basic Programming Problems)}~\citep{Austin2021ProgramSW}: A programming benchmark consisting of short problem statements and test cases.

\section{Case Study}
\label{sec:appendix_case}
We present a representative example using task 0 from the HumanEval dataset to highlight how long edit streams can degrade reasoning and code-generation quality for \textsc{AlphaEdit} and \textsc{MEMIT}, while \textbf{RLSEdit} preserves this capability.

\begin{figure*}[t]
\centering
\setlength{\tabcolsep}{8pt}
\renewcommand{\arraystretch}{1.0}

\begin{tcolorbox}[
  colback=white,
  colframe=black!35,
  arc=6pt,
  boxrule=0.8pt,
  left=8pt,right=8pt,top=6pt,bottom=6pt,
  title=\textbf{HumanEval Task 0 Prompt (\texttt{has\_close\_elements})},
  coltitle=black,
  colbacktitle=black!6
]
\textit{Task description:} Given a list of real numbers and a threshold, determine whether there exist two \emph{distinct} elements whose absolute difference is \emph{strictly less} than the threshold.
\vspace{3pt}
\begin{lstlisting}[style=promptcode,language=Python]
from typing import List

def has_close_elements(numbers: List[float], threshold: float) -> bool:
    """
    Check if in given list of numbers, are any two numbers closer to each other than
    given threshold.
    >>> has_close_elements([1.0, 2.0, 3.0], 0.5)
    False
    >>> has_close_elements([1.0, 2.8, 3.0, 4.0, 5.0, 2.0], 0.3)
    True
    """
\end{lstlisting}
\end{tcolorbox}

\vspace{6pt}

\begin{tabular}{@{}p{0.32\textwidth}p{0.32\textwidth}p{0.32\textwidth}@{}}

\begin{tcolorbox}[casepanelGood,casefixed,title={Pre-edit model \textcolor{white}{(Correct)}}]
\casehintgood{Baseline output is correct (distinct pairs and strict inequality \texttt{<}).}
\begin{lstlisting}[style=casecode,language=Python]
from typing import List

def has_close_elements(numbers: List[float], threshold: float) -> bool:
    for i in range(len(numbers)):
        for j in range(i + 1, len(numbers)):
            if abs(numbers[i] - numbers[j]) < threshold:
                return True
    return False
\end{lstlisting}
\end{tcolorbox}
&
\begin{tcolorbox}[casepanelGood,casefixed,title={AlphaEdit @ \ckpt{2k} edits \textcolor{white}{(Correct)}}]
\casehintgood{Uses distinct pairs ($j=i+1$) and strict inequality (\texttt{<}).}
\begin{lstlisting}[style=casecode,language=Python]
from typing import List

def has_close_elements(numbers: List[float], threshold: float) -> bool:
    for i in range(len(numbers)):
        for j in range(i + 1, len(numbers)):
            if abs(numbers[i] - numbers[j]) < threshold:
                return True
    return False
\end{lstlisting}
\end{tcolorbox}
&
\begin{tcolorbox}[casepanelBad,casefixed,title={AlphaEdit @ \ckpt{4k} edits (Boundary error)}]
\casehintbad{Uses \texttt{<=} instead of \texttt{<}. Counterexample: \texttt{[0.0, 0.5]}, threshold=0.5.}
\begin{lstlisting}[style=casecode,language=Python]
from typing import List

def has_close_elements(numbers: List[float], threshold: float) -> bool:
    for i in range(len(numbers)):
        for j in range(i + 1, len(numbers)):
            if abs(numbers[i] - numbers[j]) <= threshold:
                return True
    return False
\end{lstlisting}
\end{tcolorbox}
\\[-2pt]

\begin{tcolorbox}[casepanelBad,casefixed,title={AlphaEdit @ \ckpt{6k} edits (Boundary error)}]
\casehintbad{Same boundary error as 4k.}
\begin{lstlisting}[style=casecode,language=Python]
from typing import List

def has_close_elements(numbers: List[float], threshold: float) -> bool:
    for i in range(len(numbers)):
        for j in range(i + 1, len(numbers)):
            if abs(numbers[i] - numbers[j]) <= threshold:
                return True
    return False
\end{lstlisting}
\end{tcolorbox}
&
\begin{tcolorbox}[casepanelBad,casefixed,title={AlphaEdit @ \ckpt{8k} edits (Semantic bug)}]
\casehintbad{Loops allow $j=i$, so \texttt{abs(x-x)=0} and it returns True spuriously for any threshold $>0$.}
\begin{lstlisting}[style=casecode,language=Python]
def has_close_elements(numbers, threshold):
    for i in range(len(numbers)):
        for j in range(len(numbers)):
            if abs(numbers[i] - numbers[j]) < threshold:
                return True
    return False
\end{lstlisting}
\end{tcolorbox}
&
\begin{tcolorbox}[casepanelBad,casefixed,title={AlphaEdit @ \ckpt{10k} edits (Garbled / empty output)}]
\casehintbad{Non-executable output (near-empty / whitespace / escape sequences).}
\begin{lstlisting}[style=casecode]
\n\n\n\n\n\n\n\n\n\n\n\n\n\n\n\n\n
\end{lstlisting}
\end{tcolorbox}

\end{tabular}

\vspace{3pt}
\caption{\textbf{Case study on HumanEval task 0 (AlphaEdit).}
The prompt above is the original statement of HumanEval/0.
\textsc{AlphaEdit} remains correct at 2k edits but begins to fail from 4k edits onward (boundary error), later exhibiting a semantic bug at 8k and degenerating into near-empty/garbled output at 10k.}
\label{fig:case_humaneval0_alpha}
\end{figure*}

\begin{figure*}[t]
\centering
\setlength{\tabcolsep}{8pt}
\renewcommand{\arraystretch}{1.0}

\begin{tabular}{@{}p{0.32\textwidth}p{0.32\textwidth}p{0.32\textwidth}@{}}

\begin{tcolorbox}[casepanelGood,casefixed,title={Pre-edit model \textcolor{white}{(Correct)}}]
\begin{lstlisting}[style=casecode,language=Python]
from typing import List

def has_close_elements(numbers: List[float], threshold: float) -> bool:
    for i in range(len(numbers)):
        for j in range(i + 1, len(numbers)):
            if abs(numbers[i] - numbers[j]) < threshold:
                return True
    return False
\end{lstlisting}
\end{tcolorbox}
&
\begin{tcolorbox}[casepanelGood,casefixed,title={RLSEdit @ \ckpt{2k} edits \textcolor{white}{(Correct)}}]
\begin{lstlisting}[style=casecode,language=Python]
from typing import List

def has_close_elements(numbers: List[float], threshold: float) -> bool:
    
    for i in range(len(numbers)):
        for j in range(i + 1, len(numbers)):
            if abs(numbers[i] - numbers[j]) < threshold:
                return True
    return False
\end{lstlisting}
\end{tcolorbox}
&
\begin{tcolorbox}[casepanelGood,casefixed,title={RLSEdit @ \ckpt{4k} edits \textcolor{white}{(Correct)}}]
\begin{lstlisting}[style=casecode,language=Python]
from typing import List

def has_close_elements(numbers: List[float], threshold: float) -> bool:
    for i in range(len(numbers)):
        for j in range(i + 1, len(numbers)):
            if abs(numbers[i] - numbers[j]) < threshold:
                return True
    return False

\end{lstlisting}
\end{tcolorbox}
\\[-2pt]

\begin{tcolorbox}[casepanelGood,casefixed,title={RLSEdit @ \ckpt{6k} edits \textcolor{white}{(Correct)}}]
\begin{lstlisting}[style=casecode,language=Python]
from typing import List

def has_close_elements(numbers: List[float], threshold: float) -> bool:
    for i in range(len(numbers)):
        for j in range(i + 1, len(numbers)):
            if abs(numbers[i] - numbers[j]) < threshold:
                return True
    return False

\end{lstlisting}
\end{tcolorbox}
&
\begin{tcolorbox}[casepanelGood,casefixed,title={RLSEdit @ \ckpt{8k} edits \textcolor{white}{(Correct)}}]
\begin{lstlisting}[style=casecode,language=Python]
from typing import List

def has_close_elements(numbers: List[float], threshold: float) -> bool:
    for i in range(len(numbers)):
        for j in range(i + 1, len(numbers)):
            if abs(numbers[i] - numbers[j]) < threshold:
                return True
    return False

\end{lstlisting}
\end{tcolorbox}
&
\begin{tcolorbox}[casepanelGood,casefixed,title={RLSEdit @ \ckpt{10k} edits \textcolor{white}{(Correct)}}]
\begin{lstlisting}[style=casecode,language=Python]
from typing import List

def has_close_elements(numbers: List[float], threshold: float) -> bool:
    for i in range(len(numbers)):
        for j in range(i + 1, len(numbers)):
            if abs(numbers[i] - numbers[j]) < threshold:
                return True
    return False
\end{lstlisting}
\end{tcolorbox}

\end{tabular}

\vspace{3pt}
\caption{\textbf{Case study on HumanEval task 0 (RLSEdit).}
In contrast to \textsc{AlphaEdit}, \textbf{RLSEdit} preserves a correct implementation across all checkpoints (2k--10k).}
\label{fig:case_humaneval0_rls}
\end{figure*}

\clearpage
\begin{figure*}[t]
\centering
\setlength{\tabcolsep}{8pt}
\renewcommand{\arraystretch}{1.0}

\begin{tabular}{@{}p{0.32\textwidth}p{0.32\textwidth}p{0.32\textwidth}@{}}

\begin{tcolorbox}[casepanelGood,casefixed,title={Pre-edit model \textcolor{white}{(Correct)}}]
\begin{lstlisting}[style=casecode,language=Python]
from typing import List

def has_close_elements(numbers: List[float], threshold: float) -> bool:
    for i in range(len(numbers)):
        for j in range(i + 1, len(numbers)):
            if abs(numbers[i] - numbers[j]) < threshold:
                return True
    return False
\end{lstlisting}
\end{tcolorbox}
&
\begin{tcolorbox}[casepanelBad,casefixed,title={MEMIT @ \ckpt{2k} edits (Garbled)}]
\begin{lstlisting}[style=casecode]
Barrett\ufffd\ufffd\ufffd S Japan; Italian\ufffd:// Shea://\ufffd\ufffd Japan cath Italy\ufffd R Shea:// Shea Shea:// Japan Ne Shea Shea Ne Barcelona\ufffd Ne Belgium\ufffd Japan (://\ufffd\ufffd road Shea Shea:// B Shea Shea Ne://://\ufffd Japan ( Belgium\ufffd Belgium cath cath Belgium\ufffd Belgium Belgium:// Belgium Ne Belgian cath Tokyo Tokyo ( Belgium\ufffd Belgium Belgium://e B\ufffd://\ufffdanced\ufffd Ballet\ufffd Italy Ne Belgian Hub Shea
...
\end{lstlisting}
\end{tcolorbox}
&
\begin{tcolorbox}[casepanelBad,casefixed,title={MEMIT @ \ckpt{4k} edits (Garbled)}]
\begin{lstlisting}[style=casecode]
hemhemhemhem jazzhemhem jazzhem jazzhem jazzhem jazzhem jazzhem jazz jazzhem jazz jazzhem jazz jazzhem jazz jazz jazz jazz jazz jazz jazz jazz jazz jazz jazz jazz jazz jazz jazz jazz jazz jazz jazz jazz jazz jazz jazz jazz jazz jazz jazz jazz jazz jazz 
...
\end{lstlisting}
\end{tcolorbox}
\\[-2pt]

\begin{tcolorbox}[casepanelBad,casefixed,title={MEMIT @ \ckpt{6k} edits (Garbled)}]
\begin{lstlisting}[style=casecode]
SGlobal onSGlobal VictoriaongSGlobal Victoria onenn348.usermodel.usermodelhem onenn onenn like\ufffdSGlobaltee Victoria VictoriaSGlobal Victoria onSGlobal Victoria VictoriaSGlobal Victoria VictoriaSGlobal Victoria onSGlobal onSGlobal onSGlobal onSGlobal onSGlobal onSGlobal onSGlobal onSGlobalinesSGlobal onSGlobal onSGlobal onSGlobalines348 Robbieines onSGlobal Victoria Victoria VictoriaSGlobal Victoria onSGlobalinesines348 Rob
...
\end{lstlisting}
\end{tcolorbox}
&
\begin{tcolorbox}[casepanelBad,casefixed,title={MEMIT @ \ckpt{8k} edits (Empty)}]
\begin{lstlisting}[style=casecode]
No Output
\end{lstlisting}
\end{tcolorbox}
&
\begin{tcolorbox}[casepanelBad,casefixed,title={MEMIT @ \ckpt{10k} edits (Empty)}]
\begin{lstlisting}[style=casecode]
No Output
\end{lstlisting}
\end{tcolorbox}

\end{tabular}

\vspace{3pt}
\caption{\textbf{Case study on HumanEval task 0 (MEMIT).}
Under long edit streams, \textsc{MEMIT} quickly degenerates into non-executable, garbled or empty text outputs across checkpoints, unlike \textbf{RLSEdit} which preserves a valid implementation.}
\label{fig:case_humaneval0_memit}
\end{figure*}